\documentclass[twoside]{article}

\usepackage[accepted]{aistats2022}
\usepackage[authoryear,round]{natbib}

\usepackage{macros}

\renewcommand{\epsilon}{\varepsilon}

\newcommand{\ucb}{\textsf{UCB}\xspace}
\newcommand{\corral}{\textsf{Corral}\xspace}

\newcommand{\smoothCorral}{\textsf{Smooth Corral}\xspace}

\newcommand{\linucb}{\textsf{LinUCB}\xspace}
\newcommand{\suplinucb}{\textsf{SupLinUCB}\xspace}
\newcommand{\linucbOracle}{\textsf{LinUCB Oracle}\xspace}
\newcommand{\linucbPlus}{\textsf{LinUCB}{\tt++}\xspace}

\newcommand{\linucbPlusCorral}{\textsf{LinUCB}{\tt++} \textsf{with Corral}\xspace}
\newcommand{\dynamicBalancing}{\textsf{Dynamic Balancing}\xspace}

\DeclareMathOperator{\X}{\mathcal{X}}
\DeclareMathOperator{\A}{\mathcal{A}}
\DeclareMathOperator{\F}{\mathcal{F}}

\begin{document}

%

%

\twocolumn[

\aistatstitle{Pareto Optimal Model Selection in Linear Bandits}

\aistatsauthor{ Yinglun Zhu  \And  Robert Nowak }

\aistatsaddress{ University of Wisconsin-Madison \And University of Wisconsin-Madison }  ]

\begin{abstract}
We study model selection in linear bandits, where the learner must adapt to the dimension (denoted by $d_\star$) of the smallest hypothesis class containing the true linear model while balancing exploration and exploitation. Previous papers provide various guarantees for this model selection problem, but have limitations; i.e., the analysis requires favorable conditions that allow for inexpensive statistical testing to locate the right hypothesis class or are based on the idea of ``corralling'' multiple base algorithms, which often performs relatively poorly in practice. These works also mainly focus on upper bounds. In this paper, we establish the first lower bound for the model selection problem. Our lower bound implies that, even with a fixed action set, adaptation to the unknown dimension $d_\star$ comes at a cost: There is no algorithm that can achieve the regret bound $\widetilde{O}(\sqrt{d_\star T})$ simultaneously for all values of $d_\star$. We propose Pareto optimal algorithms that match the lower bound. Empirical evaluations show that our algorithm enjoys superior performance compared to existing ones.
\end{abstract}

\section{INTRODUCTION}
\label{sec:intro}

Model selection considers the problem of choosing an appropriate hypothesis class to conduct learning, and the hope is to optimally balance two types of error: the approximation error and the estimation error. In the supervised learning setting, the learner is provided with a (usually nested) sequence of hypothesis classes $\cH_d \subset \cH_{d+1}$. As an example, $\cH_d$ could be the hypothesis class consisting of polynomials of degree at most $d$. The goal is to design a learning algorithm that adaptively selects the best of these hypothesis classes, denoted by $\cH_\star$, to optimize the trade-off between approximation error and estimation error. Structural Risk Minimization (SRM) \citep{vapnik1974theory, vapnik1995nature, shawe1998structural} provides a principled way to conduct model selection in the standard supervised learning setting. SRM can automatically adapt to the complexity of the hypothesis class $\cH_\star$, with only additional logarithmic factors in sample complexity. Meanwhile, cross-validation \citep{stone1978cross, craven1978smoothing, shao1993linear} serves as a helpful tool to conduct model selection in practice.

Despite the importance and popularity of model selection in the supervised learning setting, only very recently have researchers started to study on model selection problems in interactive/sequential learning setting with bandit feedback. Two additional difficulties are highlighted in such bandit setting \citep{foster2019model}: (1) decisions/actions must be made online/sequentially without seeing the entire dataset; and (2) the learner's actions influence what data is observed, i.e., we only have partial/bandit feedback. In the simpler online learning setting with full information feedback, model selection results analogous to those in the supervised learning setting are obtained by several parameter-free online learning algorithms \citep{mcmahan2013minimax, orabona2014simultaneous, koolen2015second, luo2015achieving, orabona2016coin, foster2017parameter, cutkosky2017online, cutkosky2018black}.

The model selection problem for (contextual) linear bandits is first introduced by \cite{foster2019model}. They consider a sequence of nested linear classifiers in $\R^{d_i}$ as the set of hypothesis classes, with $d_1 < d_2 < \cdots < d_M = d$. The goal is to adapt to the smallest hypothesis class, with apriori \emph{unknown} dimension $d_\star$, that  preserves linearity in rewards. Equivalently, one can think of the model selection problem as learning a true reward parameter $\theta_\star \in \R^d$, but only the first $d_\star$ entries of $\theta_\star$ contain non-zero values. The goal is to design algorithms that could automatically adapt to the intrinsic dimension $d_\star$, rather than suffering the ambient dimension $d$. In favorable scenarios when one can cheaply test linearity, \cite{foster2019model} provide an algorithm with regret guarantee that scales as $\widetilde{O}(K^{1/4} T^{3/4}/\gamma^2 + \sqrt{K d_{\star} T }/\gamma^4)$, where $K$ is the number of arms and $\gamma$ is the smallest eigenvalue of the expected design matrix. The core idea therein is to conduct a sequential test, with sublinear sample complexity, to determine whether to step into a larger hypothesis class on the fly. Although this provides the first guarantee for model selection in the linear bandits, the regret bound is proportional to the number of arms $K$ and the reciprocal of the smallest eigenvalue, i.e., $\gamma^{-1}$. Both $K$ and $\gamma^{-1}$ can be quite large in practice, thus limiting the application of their algorithm. Recall that, when provided with the optimal hypothesis class, the classical algorithm \linucb \citep{chu2011contextual, auer2002using} for linear bandit achieves a regret bound $\widetilde{O}(\sqrt{d_\star T})$, with only polylogarithmic dependence on $K$ and no dependence on $\gamma^{-1}$.

The model selection problem in linear bandits was further studied in many subsequent papers. We roughly divide these methods into the following two sub-categories:
\begin{enumerate}[leftmargin=.2in]
    \item \textbf{Testing in Favorable Scenarios.} The algorithm in \cite{ghosh2020problem} conducts a sequence of statistical tests to gradually estimate the true support (non-zero entries) of $\theta_\star$, and then applies standard linear bandit algorithms on identified support. The regret bound of their algorithm scales as $\widetilde{O}(d^2/\gamma^{4.65} + d^{1/2}_{\star}T^{1/2})$, where $\gamma = \min\{|\theta_{\star,i}|:\theta_{\star, i} \neq 0\}$ is the minimum magnitude of non-zero entries in $\theta_\star$. Their regret bound not only depends on the ambient dimension $d$ but also scales inversely proportional to a small quantity $\gamma$. Their guarantee becomes vacuous when $d$ and/or $\gamma^{-1}$ are large. \cite{chatterji2020osom} consider a different model selection problem where the rewards come from either a linear model or a model with $K$ independent arms. Their algorithm also relies on sequential statistical testing, which requires assumptions stronger than the ones used in \cite{foster2019model} (thus suffering from similar problems).

    \item \textbf{Corralling Multiple Base Algorithms.} Another approach maintains multiple base learners and use a master algorithm to determine sample allocation among base learners. This type of algorithm is initiated by the \corral algorithm \citep{agarwal2017corralling}. Focusing on our model selection setting, the base learners are usually constructed using standard linear bandit algorithms with respect to different hypothesis classes (dimensions). To give an example of the \corral-type of algorithm, the \smoothCorral algorithm developed in \cite{pacchiano2020model} enjoys regret guarantees $\widetilde{O}(d_{\star}\sqrt{T})$ or $\widetilde{O}(d^{1/2}_{\star}T^{2/3})$. Other algorithms of this type, including some concurrent works, can be found in \cite{abbasi2020regret, arora2020corralling, pacchiano2020regret, cutkosky2020upper, cutkosky2021dynamic}.

\end{enumerate}

Note that above algorithms either only work in favorable scenarios when some critical parameters, e.g., $\gamma^{-1}$ and $K$, are not too large or must balance over multiple base algorithms which often hurts the empirical performance. They also mainly focus on developing upper bounds for the model selection problem in linear bandits.
In this paper, we explore the fundamental limits (lower bounds) of the model selection problem and design algorithms with matching guarantees (upper bounds). We establish a lower bound, using only a fixed action set, indicating that adaptation to the unknown intrinsic dimension $d_\star$ comes at a cost: There is no algorithm that can achieve the regret bound $\widetilde{O}(\sqrt{d_\star T})$ simultaneously for all values of $d_\star$. 
We also develop a Pareto optimal algorithm, with ideas fundamentally different from ``testing'' \citep{foster2019model, ghosh2020problem} and ``corralling'' \citep{pacchiano2020model, agarwal2017corralling}, to bear on the model selection problem in linear bandits. 
Our algorithm is built upon the construction of virtual mixture-arms, which is previously studied in continuum-armed bandits \citep{hadiji2019polynomial} and $K$-armed bandits \citep{zhu2020regret}. We adapt their methods to our setting, with new techniques developed to deal with the linear structure, e.g., the construction of virtual dimensions.

\subsection{Contribution and Outline}
We briefly summarize our contributions as follows.
\begin{itemize}[leftmargin=.2in]
    \item We review the model selection problem in linear bandits, and additionally define a new parameter (in \cref{sec:setting}) that reflects the tension between time horizon and the intrinsic dimension. This parameter provides a convenient way to analyze high-dimensional linear bandits.

    \item We establish the first lower bound for the model selection problem in \cref{sec:lower_bound}. Our lower bound indicates that the model selection problem is strictly harder than the problem with given optimal hypothesis class: There is no algorithm that can achieve the non-adaptive $\widetilde{O}(\sqrt{d_\star T})$ regret bound simultaneously for all values of $d_\star$. We additionally characterize the exact Pareto frontier of the model selection problem.
    
    \item In \cref{sec:adaptivity}, we develop a Pareto optimal algorithm that is fundamentally different from existing ones relying on ``testing'' or ``corralling''. Our algorithm is built on the construction of virtual mixture-arms and virtual dimensions. Although our main algorithm is analyzed under a mild assumption, we also provide a workaround.
    
    \item We conduct experiments in \cref{sec:experiment} to evaluate our algorithms. Our main algorithm shows superior performance compared to existing ones. We also show that our main algorithm is fairly robust to the existence of the assumption used in our analysis.
\end{itemize}

\subsection{Other Related Work}

\paragraph{Bandit with Large/Continuous Action Space.} Adaptivity issues naturally arises in bandit problems with large or infinite action space. In continuum-armed bandit problems \citep{agrawal1995continuum}, actions are embedded into a bounded subset $\X \subseteq \R^d$ with a smooth function $f$ governing the mean payoff for each arm. Achievable theoretical guarantees are usually influenced by some smoothness parameters, and an important question is to design algorithms that adapt to these \emph{unknown} parameters, as discussed in \cite{bubeck2011lipschitz}. \cite{locatelli2018adaptivity} show that, however, no strategy can be optimal simultaneously over all smoothness classes. \cite{hadiji2019polynomial} establishes the Pareto frontier for continuum-armed bandits with H\"older reward functions. 
Adaptivity is also studied in the discrete case with a large action space \citep{wang2008algorithms, lattimore2015pareto, chaudhuri2018quantile, russo2018satisficing, zhu2020regret}. \citet{lattimore2015pareto} studies the Pareto frontier in standard $K$-armed bandits. \citet{zhu2020regret} develop Pareto optimal algorithms for the case with multiple best arms.

\paragraph{High-Dimensional Linear Bandits.} As more and more complex data are being used and analyzed, modern applications of linear bandit algorithms usually involve dealing with ultra-high-dimensional data, sometimes with dimension even larger than time horizon \citep{deshpande2012linear}. 
To make progress in this high-dimensional regime, one natural idea is to study (or assume) sparsity in the reward vector and try to adapt to the unknown true support (non-zero entries). The sparse bandit problem is strictly harder than the model selection setting considered here due to the absence of the hierarchical structures. Consequently, a lower bound on the regret of the form $\Omega(\sqrt{d T})$, which scales with the ambient dimension $d$, is indeed unavoidable in the sparse linear bandit problem \citep{abbasi2012online, lattimore2020bandit}. Other papers deal with the sparsity setting with additional feature feedback \citep{oswal2020linear} or further distributional/structual assumptions \citep{carpentier2012bandit, hao2020high} to circumvent the lower bound. These high-dimensional linear bandit problems motivate our investigation of the relationship between time horizon and data dimension.

\section{PROBLEM SETTING}
\label{sec:setting}

We consider a linear bandit problem with a finite action set $\cA \subseteq \R^d$ where $\abs*{\cA} = K$ \citep{auer2002using, chu2011contextual}. (The feature representation of) Each arm/action $a \in \cA$ is viewed as a $d$ dimensional vector, and its expected reward $f(a)$ is linear with respect to a reward parameter $\theta_\star \in \R^d$, i.e., $f(a) = \ang*{a, \theta_\star}$. As standard in the literature \citep{ lattimore2020bandit}, we assume $\max_{a \in \cA} \| a \| \leq 1$ and $\| \theta_\star \| \leq 1$. The bandit instance is said to have intrinsic dimension $d_\star$ if $\theta_\star$ only has non-zero entries on its first $d_\star \leq d$ coordinates. The model selection problem aims at designing algorithm that can automatically adapt to the \emph{unknown} intrinsic dimension $d_\star$ in the interactive learning setting with bandit feedback.

At each time step $t \in [T]$,\footnote{Throughout the paper, we denote $[n]=\{1,2,\dots, n\}$ for any positive integer $n$.} the algorithm selects an action $A_t \in {\cal A}$ based on previous observations and receives a reward $X_t = \langle A_t, \theta_\star \rangle + \eta_t$, where $\eta_t$ is an independent $1$-sub-Gaussian noise. We define the pseudo regret (which is random, due to randomness in $A_t$) over time horizon $T$ as $\widehat{R}_T = \sum_{t=1}^T \left\langle \theta_\star,  a_\star - {A_t} \right\rangle$, where $a_\star$ corresponds to the best action in action set, i.e., $a_\star = \argmax_{a \in {\cal A}} \langle a, \theta_\star \rangle$. We measure the performance of any algorithm by its expected regret $R_{T} = \E \sq*{ \widehat{R}_T } = \E \sq*{ \sum_{t=1}^T \left\langle \theta_\star,  a_\star - {A_t} \right\rangle }$.

We primarily focus on the high-dimensional linear bandit setting with ambient dimension $d$ close to or even larger than (the allowed) time horizon $T$. We use ${\cal R}(T, d_\star)$ to denote the set of regret minimization problems with time horizon $T$ and any bandit instance with intrinsic dimension $d_\star$. We emphasize that $T$ is part of the problem instance, which was largely neglected in previous work focusing on the low dimensional regime where $T \gg d_\star$. To model the tension between the allowed time horizon and the intrinsic dimension, we define the hardness level as 
\begin{align*}
    \psi \left({\cal R}(T, d_\star) \right) = \min \{ \alpha \geq 0: d_\star \leq T^{\alpha}  \} = \log d_\star/ \log T.
\end{align*}
$\psi({\cal R}(T, d_\star))$ is used here since it precisely captures the regret over the set of regret minimization problem ${\cal R}(T, d_\star)$, as discussed later in our review of the \linucb algorithm and the lower bound. Since smaller $\psi({\cal R}(T, d_\star))$ indicates easier problem, we define the family of regret minimization problems with \emph{hardness level} at most $\alpha$ as 
\begin{align}
	{\cal H}_T(\alpha) = \{ \cup {\cal R}(T, d_\star) : \psi({\cal R}(T, d_\star))\leq \alpha \} \nonumber,
\end{align}
where $\alpha \in [0, 1]$. Although $T$ is necessary to define a regret minimization problem, the hardness of the problem is encoded into a single parameter $\alpha$: Problems with different time horizons but the same $\alpha$ are equally difficult in terms of the regret achieved by \linucb (the exponent of $T$). 
We explore the connection $d_\star \leq T^{\alpha}$ in the rest of this paper and focus on (polynomial) dependence on $T$ (i.e., the dependence on $d_\star$ is translated into the dependence on $T^\alpha$).
We are interested in designing algorithms with worst case guarantees over ${\cal H}_T(\alpha)$, but \emph{without} the knowledge of $\alpha$.

\paragraph{\linucb and Upper Bounds.} In the standard setting where $d_\star$ is known, \linucb \cite{chu2011contextual, auer2002using} achieves $\widetilde{O}(\sqrt{d_\star T})$ regret.\footnote{Technically, the regret bound is only achieved by a more complicated algorithm \suplinucb. However, it's common to use \linucb as the practical algorithm. See \citet{chu2011contextual} for detailed discussion.} For any problem in ${\cal H}_T(\alpha)$ with \emph{known} $\alpha$, one could run \linucb on the first $\floor*{T^{\alpha}}$ coordinates and achieve $\widetilde{O}(T^{(1+\alpha)/2})$ regret. The goal of model selection is to achieve the $\widetilde{O}(T^{(1+\alpha)/2})$ regret but without the knowledge of $\alpha$.

 \paragraph{Lower Bounds.} In the case when $d_\star \leq \sqrt{T}$, \cite{chu2011contextual} prove a $\Omega(\sqrt{d_\star T})$ lower bound for linear bandits. When $d_\star \geq \sqrt{T}$ is the case, a lower bound $\Omega(K^{1/4} T^{3/4})$ is developed in \cite{abe2003reinforcement}.

\section{LOWER BOUND AND PARETO OPTIMALITY}
\label{sec:lower_bound}

We study lower bounds for model selection in this section. 
We show that simultaneously adapting to all hardness levels is impossible. Such fundamental limitation leads to the established of Pareto frontier.

Our lower bound is constructed by relating the regrets between two (sets of) closely related problems: We show that any algorithm achieves good performance on one of them necessarily performs bad on the other one. 
Similar ideas are previously explored in continuum-armed bandit and $K$-armed bandits \citep{locatelli2018adaptivity, hadiji2019polynomial, zhu2020regret}.
We study the linear case with model selection and establish the following lower bound.\footnote{Our lower bound is quantitatively similar to the one studied in $K$-armed bandits with multiple best arms \citep{zhu2020regret}.}
We use $\omega \in {\cal H}_T(\alpha)$ to represent any bandit regret minimization problem with time horizon $T$ and hardness level at most $\alpha$ (i.e., $d_\star \leq T^\alpha$).

\begin{restatable}{theorem}{lowerBound}
	\label{thm:lower_bound}
	Consider any $0 \leq \alpha^\prime < \alpha \leq 1$ and $B > 0$ satisfying $T^{\alpha} \leq B$ and $\lfloor T^\alpha/2 \rfloor \geq \max\{ T^\alpha/4, T^{\alpha^\prime}, 2\}$. If an algorithm is such that $\sup_{\omega \in {\cal H}_T(\alpha^\prime)} R_T \leq B$, then the regret of the same algorithm must satisfy
	\begin{align}
	\sup_{\omega \in {\cal H}_T(\alpha)} R_T \geq 2^{-10} \, T^{1+\alpha} B^{-1} .\label{eq:lower_bound}
	\end{align}
\end{restatable}

Our lower bound delivers important messages to the model selection problem in linear bandits. Most of the previous efforts and open problems \citep{foster2019model, pacchiano2020model} are made to match the usual non-adaptive regret with known $d_\star$ (or $\alpha$). Our lower bound, however, provides a negative answer towards the open problem of achieving regret guarantees $\widetilde{O}(T^{(1+\alpha)/2})$ simultaneously for all hardness levels $\alpha$. 
We interpret this result next.

\paragraph{Interpretation of \cref{thm:lower_bound}.} Fix any linear bandit algorithm. We consider two problem instances with different hardness levels $0 \leq \alpha^\prime < \alpha \leq 1$ (and satisfy the constrains in \cref{thm:lower_bound}). On one hand, if the algorithm is such that $\sup_{\omega \in {\cal H}_T(\alpha^\prime)} R_T = \widetilde{\omega}(T^{(1+\alpha^\prime)/2})$, we know that this algorithm is already sub-optimal over problems with hardness level at most $\alpha^\prime$. On the other hand, suppose that the algorithm achieves the desired regret $\widetilde{O}(T^{(1+\alpha^\prime)/2})$ over ${\cal H}_T(\alpha^\prime)$. \cref{eq:lower_bound} then tells us that $\sup_{\omega \in {\cal H}_T(\alpha)} R_T = \widetilde{\Omega}(T^{(1+2\alpha - \alpha^\prime)/2})$, which is (asymptotically) larger than the desired regret $\widetilde{O}(T^{(1+\alpha)/2})$ over problems with hardness level at most $\alpha$.

If we aim at providing regret bounds with only polylogarithmic dependence on $K$ in linear bandits (which is usually the case for linear bandits with finite action set \citep{auer2002using, chu2011contextual}). our lower bound also provides a negative answer to the open problem of achieving a weaker guarantee $\widetilde{O}(T^\gamma d_\star^{1-\gamma}) = \widetilde{O}(T^{\gamma + \alpha(1-\gamma)})$, with $\gamma \in [1/2,1)$ \citep{foster2019model}, simultaneously for all $d_\star$ (or $\alpha$).

In the model selection setting, the performance of any algorithm should be a function of the hardness level $\alpha$: The algorithm needs to adapt the \emph{unknown} $\alpha$. To further explore the fundamental limit for model selection in linear bandits, following \cite{hadiji2019polynomial, zhu2020regret}, we define rate function to capture the performance of any algorithm (in terms of its regret dependence on polynomial terms of $T$).

\begin{definition}
	\label{def:rate}
	Let $\theta:[0,1] \rightarrow [0, 1]$ denote a non-decreasing function. An algorithm achieves the rate function $\theta$ if 
	\begin{align}
		\forall \epsilon > 0, \forall \alpha \in [0, 1], \quad \limsup_{T \rightarrow \infty} \frac{\sup_{\omega \in {\cal H}_T(\alpha)} R_T}{T^{\theta(\alpha)+\epsilon}} < + \infty. \nonumber 
	\end{align}
\end{definition}

Since there may not always exist a pointwise ordering over rate functions, we consider the notion of Pareto optimality over rate functions.

\begin{figure}[h]
    \vspace{-10pt}
    \centering
    \includegraphics[width=1\linewidth]{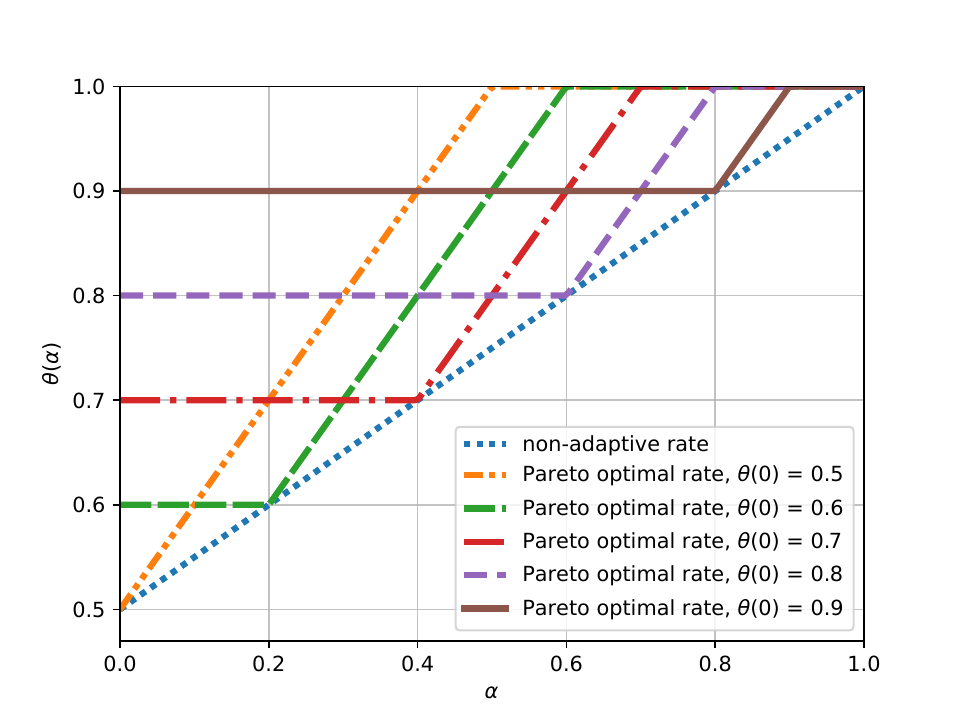}
    \caption{Pareto Optimal Rates for Model Selection in Linear Bandits.}
    \label{fig:pareto_rate}
\end{figure}

\begin{definition}
	\label{def:pareto}
	A rate function $\theta$ is Pareto optimal if it is achieved by an algorithm, and there is no other algorithm achieving a strictly smaller rate function $\theta^\prime$ in the pointwise order. An algorithm is Pareto optimal if it achieves a Pareto optimal rate function.
\end{definition}

We establish the following lower bound for any rate function that can be achieved by an algorithm designed for model selection in linear bandits.

\begin{restatable}{theorem}{thmRateLowerBound}
    \label{thm:rate_lower_bound}
Suppose a rate function $\theta$ is achieved by an algorithm, then we must have 
\begin{align}
\label{eq:rate_lower_bound}
    \theta(\alpha) \geq \min \{ \max \{ \theta(0), 1 + \alpha - \theta(0) \}, 1 \},
\end{align}
with $\theta(0) \in [1/2, 1]$.
\end{restatable}

\cref{fig:pareto_rate} illustrates the Pareto frontiers for the model selection problem in linear bandits:
The blue dashed line represents the non-adaptive rate function achieved by \linucb with \emph{known} $\alpha$;
Other curves represent Pareto optimal rate functions (achieved by Pareto optimal algorithms introduced in \cref{sec:adaptivity}) for the model selection problem in linear bandits.
\cref{fig:pareto_rate} implies that no algorithm can achieve the non-adaptive rate simultaneously for all $\alpha$: any Pareto optimal curve has to be higher than the non-adaptive curve at least at some points.

\paragraph{Pareto Optimality of \corral-Type of Algorithms.} We remark that, accompanied with our lower bound, the \smoothCorral algorithm presented in \cite{pacchiano2020model} is also Pareto optimal. While only a $\widetilde O(d_\star \sqrt{T})$ regret bound is presented for the \smoothCorral algorithm, upon inspection of their analysis, we find that \smoothCorral can actually match the lower bound in \cref{eq:rate_lower_bound} by setting the learning rate as $\eta = T^{-\theta(0)}$, for any $\theta(0) \in [1/2,1)$. 
See \cref{app:corral} for a detailed discussion.

Although the \corral-type of algorithm (e.g., \smoothCorral) is Pareto optimal, they may not be effective in problems with specific structures \citep{papini2021leveraging}. We introduce a new Pareto optimal algorithm in the next section, which is shown to be more practical than \smoothCorral regarding model selection problems in linear bandits (see \cref{sec:experiment}).

\section{PARETO OPTIMALITY WITH NEW IDEAS}
\label{sec:adaptivity}

We develop a Pareto optimal algorithm \linucbPlus (\cref{alg:linucbPlus}) that operates fundamentally different from algorithms rely on ``testing'' \citep{foster2019model, ghosh2020problem} or ``corralling'' \citep{pacchiano2020model, agarwal2017corralling}. 
Our algorithm is built upon the construction of virtual mixture-arms \citep{hadiji2019polynomial, zhu2020regret} and virtual dimensions.

We first introduce some additional notations. For any vector $a \in \R^d$ and $0 \leq d_i \leq d$, we use $a^{(d_i)} \in \R^{d_i}$ to represent the truncated version of $a$ that only keeps the first $d_i$ dimensions. We also use $[a_1; a_2]$ to represent the concatenated vector of $a_1$ and $a_2$. We denote $\A^{(d_i)} \subseteq \R^{d_i}$ as the ``truncated" action (multi-) set, i.e., ${\cal A}^{(d_i)} = \curly*{ a^{(d_i)}   \in \R^{d_i}: a \in {\cal A} }$.
One can always \emph{manually} construct the truncated action set ${\cal A}^{(d_i)}$ and \emph{pretend} to work with arms with truncated feature representations (though their expected rewards may not be aligned with the truncated feature representations).

\begin{algorithm}[]
	\caption{\linucbPlus}
	\label{alg:linucbPlus} 
	\renewcommand{\algorithmicrequire}{\textbf{Input:}}
	\renewcommand{\algorithmicensure}{\textbf{Output:}}
	\begin{algorithmic}[1]
		\REQUIRE Time horizon $T$ and a user-specified parameter $\beta \in [1/2, 1)$.
		\STATE \textbf{Set:} $p = \lceil \log_2 T^\beta \rceil$, $d_i = \min\{2^{p+2-i}, d\}$ and $\Delta T_i = \min \{ 2^{p + i}, T\}$.
		\FOR {$i = 1, \dots, p$}
		\STATE Run \linucb on a set of arms $S_i$ for $\Delta T_i$ rounds, where $S_i$ contains all arms in ${\cal A}^{(d_i)}$ \emph{and} a set of virtual mixture-arms constructed from previous iterations, i.e., $\{\widetilde{\nu}_j\}_{j < i}$. \linucb is operated with respect to an modified linear bandit problem with added virtual dimensions.
		\STATE Construct a virtual mixture-arm $\widetilde{\nu}_i$ based on empirical sampling frequencies in iteration $i$.
		\ENDFOR 
	\end{algorithmic}
\end{algorithm}

We present \linucbPlus in \cref{alg:linucbPlus}. \linucbPlus operates in iterations with geometrically increasing length, and it invokes \linucb (\suplinucb) \citep{chu2011contextual, auer2002using} with (roughly) geometrically decreasing dimensions. 
The core steps of \linucbPlus are summarized at lines 3 and 4 in \cref{alg:linucbPlus}, which consists of construction of virtual mixture-arms and virtual dimensions (the modified linear bandit problem).
We next explain in detail these two core ideas.

\paragraph{The Virtual Mixture-Arm.} After each iteration $j$, let $\widehat{p}_j$ denote the vector of empirical sampling frequencies of the arms in that iteration, i.e., the $k$-th element of $\widehat{p}_j$ is the number of times arm $k$, including all previously constructed virtual mixture-arms, was sampled in iteration $j$ divided by the total number of time steps $\Delta T_j$. The virtual mixture-arm for iteration $j$ is the $\widehat{p}_j$-mixture of the arms played in iteration $j$, denoted by $\widetilde{\nu}_j$. When \linucb samples from $\widetilde{\nu}_j$, it first draws a real arm $j_t \sim \widehat{p}_j$ with feature representation $A_t$,\footnote{If the index of another virtual mixture-arm is returned, we sample from that virtual mixture-arm until a real arm is returned.} then pull the real arm $A_t$ to obtain a reward $X_t = \langle \theta_\star, A_t \rangle + \eta_t$. The expected reward of virtual mixture-arm $\widetilde{\nu}_j$ can be expressed as $\langle \theta_\star, a_\star \rangle -  R_{\Delta T_j}/\Delta T_j$, where we use $R_{\Delta T_j}$ to denote the expected regret suffered in iteration $j$. Virtual mixture-arms $\widetilde{\nu}_j$ provide a convenient summary of the information gained in the $j$-th iterations so that we don't need to explore arms in the (effectively) $d_j$ dimensional space again.

\paragraph{Linear Bandits with Added Virtual Dimensions.} We consider the linear bandit problem in iteration $i$, where each arm in ${\cal A}^{(d_i)}$ is viewed as a vector in $\R^{d_i}$. Besides this simple truncation, we lift the feature representation of each arm into a slightly higher dimensional space to include the $i-1$ virtual mixture-arms constructed in previous iterations  (i.e., adding virtual dimensions). More specifically, we augment $i-1$ zeros to the feature representation of each truncated real arm $a \in \A^{(d_i)}$; we also view each virtual mixture-arm $\widetilde{\nu}_j$ as a $d_i +i-1$ dimensional vector $\widetilde{\nu}^{\langle d_i \rangle }_j$ with its $(d_i + j)$-th entry being $1$ and all other entries being $0$. As a result, \linucb will operate on an modified linear bandit problem with action set ${\cal A}^{\langle d_i \rangle } \subseteq  \R^{d_i + i - 1}$, where $	{\cal A}^{\langle d_i \rangle } = \curly*{ [a^{(d_i)}; 0] \in \R^{d_i + i - 1}:  a \in {\cal A} } \cup \curly*{ \widetilde{\nu}_j^{\langle d_i \rangle} }$, and $\vert {\cal A}^{\langle d_i \rangle } \vert =  K + i-1$.
Working with added virtual dimensions allows us to incorporate information stored in virtual mixture-arms without too much additional cost since $i \leq p = O(\log T)$.

\begin{remark}
    \label{rm:modification}
    Previous application of the virtual mixture-arms only works in continuum-armed bandits or $K$-armed bandits \citep{zhu2020regret, hadiji2019polynomial}, where no further modifications are needed to incorporate information stored in virtual mixture-arms. Besides the construction of the virtual dimension, we also provide another way to incorporate the virtual mixture-arms in \cref{sec:remove_assumption}. These modifications are important for the linear bandit case.
\end{remark}

\subsection{Analysis}
\label{sec:adaptive_analysis}

We first analyze \linucbPlus with the following assumption. A modified version of \linucbPlus (\cref{alg:linucbPlus_modified}) is provided in \cref{sec:remove_assumption} and analyzed without the assumption. 

\begin{assumption}
	\label{assumption:action_set}
An action set ${\cal A} \subseteq \R^d$ is expressive if we have $a^{[d_i]} = [a^{(d_i)};0] \in {\cal A}$ for any $a \in {\cal A}$ and $d_i < d$.
\end{assumption}

\cref{assumption:action_set} is naturally satisfied when certain combinatorial structure and ranking information are associated with the action set.
This is best explained with an example. Suppose the arms are consumer products and each has a subset of $d$ possible features, i.e., the arms are \emph{binary vectors} in ${\mathbb R}^d$ indicating the features of the product (the combinatorial aspect). Think of the features as being ordered from base-level features to high-end features (the ranking information). In this case, \cref{assumption:action_set} means that if a product $a \in {\cal A}$, then ${\cal A}$ also contains all products with fewer high-end features, i.e., truncations of action $a$. We also make the following two comments regarding \cref{assumption:action_set}.

\begin{enumerate}[leftmargin=.2in]
    \item The action set we used to construct the lower bound in \cref{thm:lower_bound} can be made expressive, as noted in \cref{rm:expressive_lower_bound} in \cref{app:lower_boud};
    \item Although the original version of \linucbPlus is analyzed with \cref{assumption:action_set}, it shows strong empirical performance even without such assumption (see \cref{sec:experiment}).
\end{enumerate}

Equipped with \cref{assumption:action_set}, we can replace the ``truncated'' action set $\A^{(d_i)}$ with real arms that actually exist in the action set. As a result, the linearity in rewards is preserved in the modified linear bandit problem in $\R^{d_i+i-1}$ with added virtual dimensions.
The modified linear bandit problem is associated with reward vector ${\theta}_\star^{\langle d_i \rangle } =  \sq*{\theta_\star^{(d_i)}; \widetilde{\mu}_1; \dots;  \widetilde{\mu}_{i-1}} \in \R^{d_i + i -1}$, where we use $\widetilde{\mu}_{i} = \langle \theta_\star, a_\star \rangle -  R_{\Delta T_{i}}/\Delta T_{i}$ to denote the expected reward of mixture-arm $\widetilde \nu_i$. In the $i$-th iteration of \linucbPlus, we invoke \linucb to learn reward vector ${\theta}_\star^{\langle d_i \rangle}\in \R^{d_i + i -1}$, which takes worst case regret proportional to $d_i + i -1$ instead of the ambient dimension $d$.

Since there are at most $O(\log T)$ iterations of \linucbPlus, we only need to upper bound its regret at each iteration. Suppose $S_i$ is the set of actions that \linucbPlus is working on at iteration $i$. We use $a_{S_i} = \argmax_{a \in S_i} \langle \theta_\star, a \rangle$ to denote the arm with the highest expected reward; and decompose the regret into approximation error and learning error:
\begin{align}
\label{eq:regret_decomposition}
 & R_{\Delta T_i} = \underbrace{ \E \left[ \Delta T_i \cdot \langle \theta_\star, a_\star - a_{S_i} \rangle  \right] }_{\text{expected approximation error due to the selection of $S_i$}} \\ \nonumber
&+ \underbrace{ \E \left[ \sum_{t=1}^{\Delta T_i} \langle \theta_\star, a_{S_i} - A_t \rangle\right]}_{\text{expected learning error due to the sampling rule $\{A_t\}_{t=1}^T$}}.
\end{align}

\paragraph{The Learning Error.} At each iteration $i$, \linucbPlus invokes \linucb on a linear bandit problem in $\R^{d_i +i - 1}$ for $\Delta T_i$ time steps, where $d_i$ and $\Delta T_i$ are specifically chosen such that ${d_i \, \Delta T_i} \leq \widetilde{O}(T^{2\beta})$. The learning error is then upper bounded by $\widetilde{O} (\sqrt{d_i \, \Delta T_i}) = \widetilde{O}(T^\beta)$ based on the regret bound of \linucb (the norm of reward vector $\theta_\star^{\ang{d_i}}$ increases with iteration $i$ due to added virtual dimensions, we deal with that in \cref{appendix:modified_linucb}).

\paragraph{The Approximation Error.} Let $i_\star \in [p]$ denote the largest integer such that $d_{i_\star} \geq d_\star$. For iterations $i \leq i_\star$, since $\theta_\star$ only has its first $d_\star \leq d_i$ coordinates being non-zero, we have $\max_{a \in \A^{\ang{d_i}}} \{ \ang*{\theta_\star^{\ang{d_i}}, a} \} = \ang{\theta_\star, a_\star}$ and the expected approximation error equals zero. As a result, we upper bound the expected regret for iteration $i \leq i_\star$ by its expected learning error, i.e., $R_{\Delta T_i} \leq \widetilde{O}(T^\beta)$. Now consider any iteration $i > i_\star$. Since the virtual mixture-arm $\widetilde{\nu}_{i_\star}$ is constructed by then, and its expected reward is $\widetilde{\mu}_{i_\star} = \langle \theta_\star, a_\star \rangle -  R_{\Delta T_{i_\star}}/\Delta T_{i_\star}$, we can further bound the expected approximation error by $\Delta T_i R_{\Delta T_{i_\star}}/\Delta T_{i_\star} = \widetilde{O}(T^{1+\alpha - \beta})$ (detailed in \cref{app:thm_linucbPlus}).

We now present the formal guarantees of \linucbPlus.

\begin{restatable}{theorem}{thmLinucbPlus}
	\label{thm:linucbPlus}
	Run \linucbPlus with time horizon $T$ and any user-specified parameter $\beta \in [1/2, 1)$ leads to the following upper bound on the expected regret:
	\begin{align}
	& \sup_{\omega \in {\cal H}_T(\alpha)}  R_T  \nonumber \\
	& =  O \left(  \log^{7/2} \left( KT\log T\right)  \cdot T^{\min \{\max\{\beta, 1  + \alpha - \beta \},1 \}} \right). \nonumber
	\end{align} 
\end{restatable}

The next theorem shows that \linucbPlus is Pareto optimal with \emph{any} input $\beta \in [1/2, 1)$.

\begin{restatable}{theorem}{pareto}
	\label{thm:pareto}
	The rate function achieved by \linucbPlus with any input $ \beta \in [1/2, 1)$, i.e.,
	\begin{align}
	\theta_{\beta}: \alpha \mapsto \min \{ \max \{ \beta, 1 + \alpha - \beta\}, 1\}, \label{eq:pareto_rate}
	\end{align}
	is Pareto optimal.
\end{restatable}

\subsection{Removing \cref{assumption:action_set}}
\label{sec:remove_assumption}

\cref{assumption:action_set} is used to preserve linearity when working with truncated action sets. In general, one should not expect to deal with misspecified linear bandits without extra cost: \citet{lattimore2020learning} develop a regret lower bound $\Omega(\epsilon \sqrt{d} \, T)$ for misspecified linear bandits with misspecification level $\epsilon$. The lower bound scales linearly with $T$ if there is no extra control/assumptions on the misspecified level $\epsilon$.

Going back to our algorithm, however, we notice that there is a special structure in the source of misspecifications: the virtual-mixture arms are never misspecified. We explore this fact and provide a modified version of \cref{alg:linucbPlus} (i.e., \cref{alg:linucbPlus_modified}) that works \emph{without} \cref{assumption:action_set} and is Pareto optimal. The modified algorithm is less practical since it invokes \smoothCorral as a subroutine (see \cref{sec:experiment}).

\begin{algorithm}[]
	\caption{\linucbPlusCorral}
	\label{alg:linucbPlus_modified} 
	\renewcommand{\algorithmicrequire}{\textbf{Input:}}
	\renewcommand{\algorithmicensure}{\textbf{Output:}}
	\begin{algorithmic}[1]
		\REQUIRE Time horizon $T$ and a user-specified parameter $\beta \in [1/2, 1)$.
		\STATE \textbf{Set:} $p = \lceil \log_2 T^\beta \rceil$, $d_i = \min\{2^{p+2-i}, d\}$ and $\Delta T_i = \min \{ 2^{p + i}, T\}$.
		\FOR {$i = 1, \dots, p$}
		\STATE Construct two (smoothed) base algorithms: (1) a \linucb algorithm working with action set $\cA^{(d_i)}$; and (2) a \ucb algorithm working with the set of virtual mixture-arms (if any), i.e., $\{\widetilde{\nu}_j\}_{j < i}$. Invoke \smoothCorral as the master algorithm with learning rate $\eta = 1/\sqrt{d_i \Delta T_i}$.
		\STATE Construct a virtual mixture-arm $\widetilde{\nu}_i$ based on the empirical sampling frequencies in iteration $i$.
		\ENDFOR 
	\end{algorithmic}
\end{algorithm}

We defer detailed discussion on \cref{alg:linucbPlus_modified} and \smoothCorral to \cref{app:remove_assumption}. 
We state the guarantee of \cref{alg:linucbPlus_modified} next.

\begin{restatable}{theorem}{thmLinUCBPlusModified}
\label{thm:linucbPlus_modified}
With any input $\beta \in [1/2,1)$, the rate function achieved by \cref{alg:linucbPlus_modified} (without \cref{assumption:action_set}) is Pareto optimal.
\end{restatable}

\section{EXPERIMENTS}
\label{sec:experiment}

We empirically evaluate our algorithms \linucbPlus and \linucbPlusCorral in this section. 
We find that \linucbPlus enjoys superior performance compared to existing algorithms.
Although \cref{assumption:action_set} is needed in the analysis of \linucbPlus, our experiments show that \linucbPlus is fairly robust to the existence of such assumption.

\begin{figure}[h]
     \centering
     \subfloat[]{\includegraphics[width=.25\textwidth]{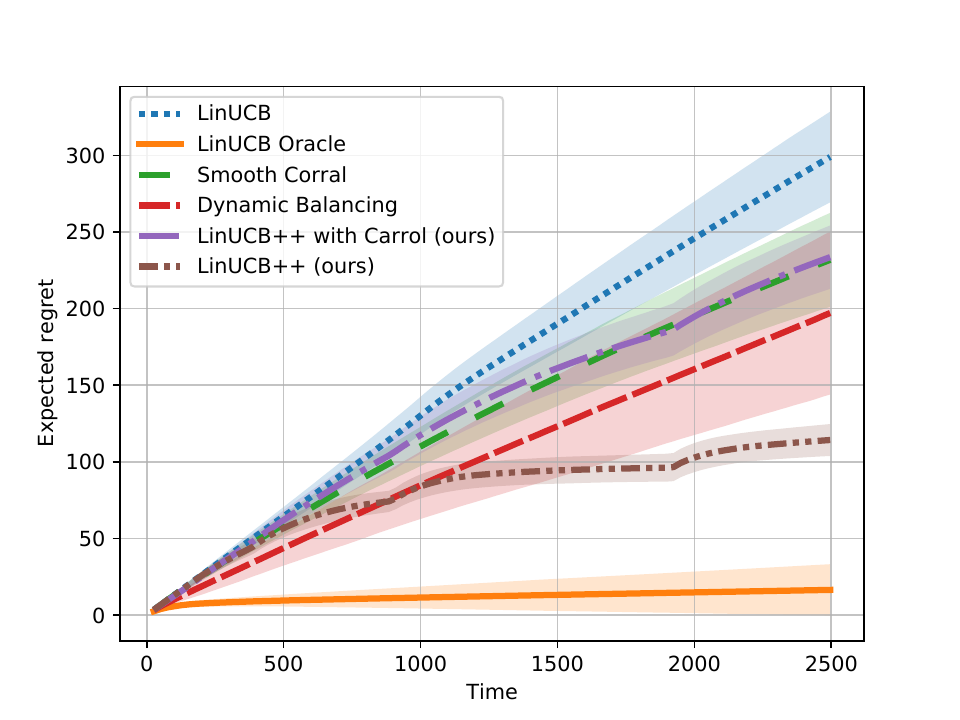}\label{fig:non_expressive_curve}}
     \subfloat[]{\includegraphics[width=.25\textwidth]{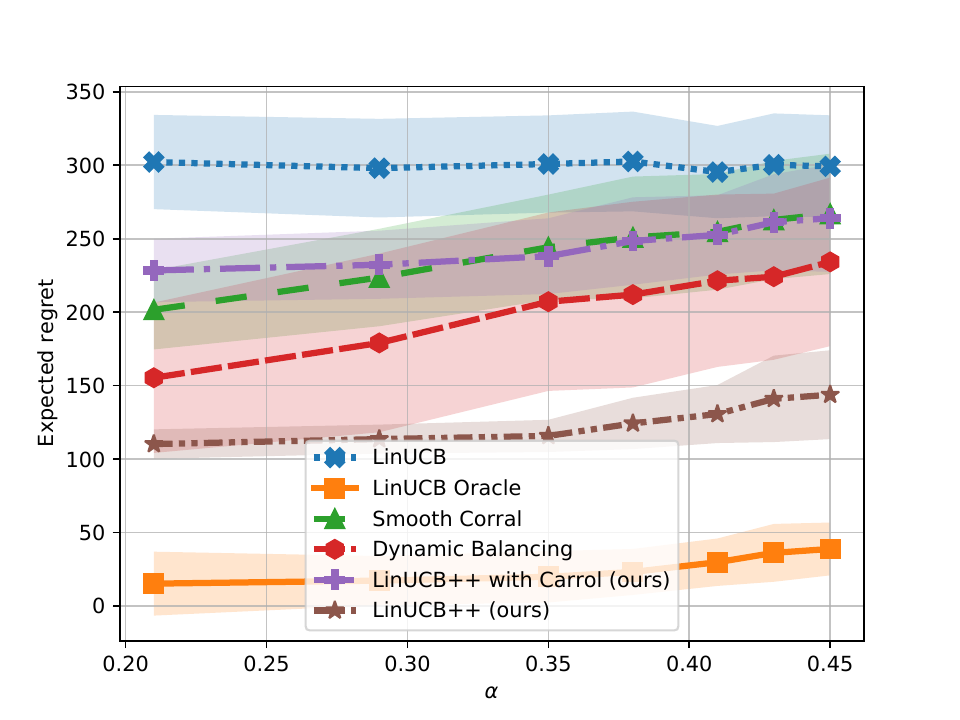}\label{fig:non_expressive_alpha}}
     \caption{Experiments \emph{without} \cref{assumption:action_set}: (a) Regret Curve Comparison with $\alpha \approx 0.32$. (b) Regret Comparison with Different $\alpha$.}
     \label{fig:non_expressive}
\end{figure}

We compare \linucbPlus and \linucbPlusCorral with four baselines: \linucb \citep{chu2011contextual}, \linucbOracle, \smoothCorral \citep{pacchiano2020model} and \dynamicBalancing \citep{cutkosky2021dynamic}.
\linucb is the standard linear bandit algorithm that works in the ambient dimension $\R^d$. \linucbOracle represents the oracle version of \linucb: it takes the knowledge of the instrinsic dimension $d_\star$ and works in $\R^{d_\star}$. \smoothCorral and \dynamicBalancing are implemented with $M = \ceil{\log_2 d}$ base \linucb learners with different dimensions $d_i \in \curly{2^0, 2^1, \dots, 2^{M-1}}$; their master algorithms conduct corraling/regret balancing on top of these base learners. 
We set $\beta = 0.5$ in \linucbPlus and \linucbPlusCorral.\footnote{In practice, we recommend taking $\beta = {(1+\widehat \alpha)}/{2}$ if an estimation $\widehat \alpha$ (of $\alpha$) is available; otherwise, we empirically find that taking $\beta = 0.5$ works well.}
The regularization parameter $\lambda$ for least squares in (all subroutines/base learners of) \linucb is set as $0.1$.

We first conduct experiments \emph{without} an expressive action set (i.e., without \cref{assumption:action_set}).
We consider a regret minimization problem with time horizon $T = 2500$ and a bandit instance consists of $K=1200$ arms selected uniformly at random in the $d=600$ dimensional unit ball. 
We set reward parameter $\theta_\star = [1/\sqrt{d_\star}, \dots, 1/\sqrt{d_\star}, 0, \dots, 0]^\top \in \R^d$ for any intrinsic dimension $d_\star$ (see \cref{app:experiment} for experiments with other choices of $\theta_\star$).
To prevent lengthy exploration over exploitation, we consider Gaussian noises with zero means and $0.1$ standard deviations.
We evaluate each algorithm on $100$ independent trials and average the results. \cref{fig:non_expressive_curve} shows how regret curves of different algorithms increase. The experiment is run with intrinsic dimension $d_\star = 12$, which corresponds to a hardness level $\alpha \approx 0.32$. \linucbPlus outperforms all other algorithms (except \linucbOracle), and enjoys the smallest variance. \linucbPlus (almost) flatten its regret curve at early stages, indicating that it has learned the true reward parameter. \cref{fig:non_expressive_alpha} illustrates the performance of algorithms with respect to different intrinsic dimensions. We run experiments with $d_\star \in \{5,10,15,20,25,30, 35\}$, and mark the corresponding $\alpha$ values in the plot. Across all $\alpha$ values, \linucbPlus shows superior performance compared to \linucb, \smoothCorral, \dynamicBalancing and \linucbPlusCorral. 
These results indicate that \linucbPlus can be practically applied without an expressive action set (thus without \cref{assumption:action_set}).

The empirically poor performance of \corral-type of algorithms might be due to the fact that they need to balance over multiple base algorithms. 
On the other hand, \linucbPlus invokes only one \linucb subroutine at each iteration. Although the subroutine is restarted at the beginning of each iteration, it runs on (roughly) geometrically decreasing dimensions. Such efficient learning procedure is backed by our construction of virtual mixture-arms and virtual dimensions.

\begin{figure}[h]
     \centering
     \subfloat[]{\includegraphics[width=.25\textwidth]{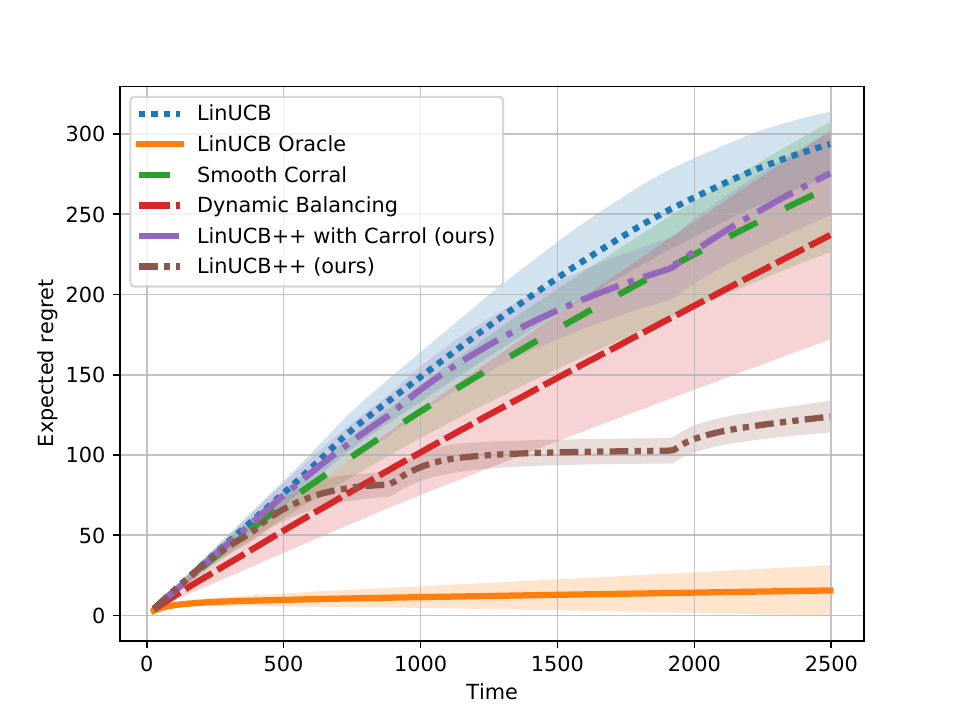}\label{fig:expressive_curve}}
     \subfloat[]{\includegraphics[width=.25\textwidth]{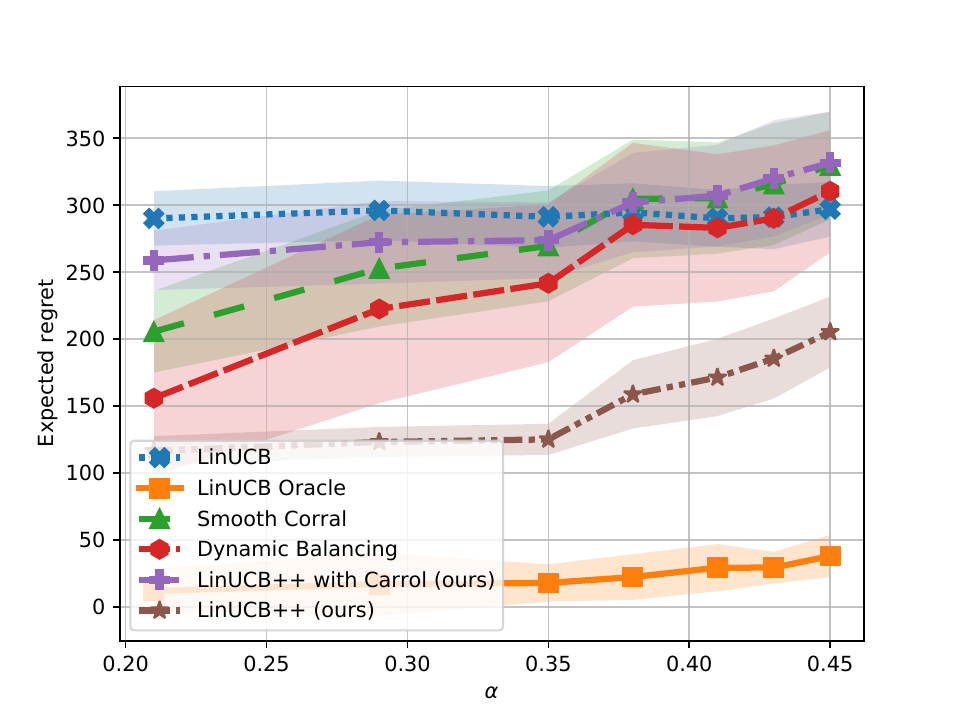}\label{fig:expressive_alpha}}
     \caption{Similar Experiment Setups to Those Shown in \cref{fig:non_expressive}, but with Expressive Action Sets.}
     \label{fig:expressive}
\end{figure}

We also run experiments \emph{with} expressive action sets.
We first generate $K=800$ arms uniformly at random from a $d=400$ dimensional unit ball. The action set is then made \emph{expressive} by adding actions with truncated features.\footnote{We only truncate actions with respect to $d_i$\,s selected by \linucbPlus to avoid the computational burden of dealing with a large number of actions.}
We provide the expressive action set to all algorithms since the best reward could be achieved by a truncated arm.
Other experimental setups are similar to the ones described before.
The shape of curves appearing in both \cref{fig:expressive_curve} and \cref{fig:expressive_alpha} are resembles the ones in \cref{fig:non_expressive}, and \linucbPlus outperforms \linucb, \smoothCorral, \dynamicBalancing and \linucbPlusCorral. One slight difference is that \smoothCorral, \dynamicBalancing, \linucbPlusCorral and \linucbPlus have relatively worse performance when as $\alpha$ increases: 
The regret curves (in \cref{fig:expressive_alpha}) increase at faster speeds.
\smoothCorral, \dynamicBalancing and \linucbPlusCorral are outperformed by the standard \linucb when the hardness level $\alpha$ gets large.

\section{DISCUSSION}
\label{sec:discussion}

We study the model selection problem in linear bandits where the goal is to adapt to the \emph{unknown} intrinsic dimension $d_\star$, rather than suffering from regret proportional to the ambient dimension $d$. We establish a lower bound indicating that adaptation to the unknown intrinsic dimension $d_\star$ comes at a cost: There is no algorithm that can achieve the regret bound $\widetilde{O}(\sqrt{d_\star T})$ simultaneously for all values of $d_\star$. Under a mild assumption, we design a Pareto optimal algorithm, with ideas fundamentally different from ``testing'' \citep{foster2019model, ghosh2020problem} and ``corralling'' \citep{pacchiano2020model, agarwal2017corralling}, to bear on the model selection problem in linear bandits. We also provide a workaround to remove the assumption. Experimental evaluations show superior performance of our main algorithm compared to existing ones.

Although linear bandits with a fixed action set are commonly studied in the literature \citep{lattimore2020learning, wagenmaker2021experimental}, an interesting direction is to generalize \linucbPlus to the contextual setting. The current version of \linucbPlus works in the setting with adversarial contexts under the following two additional assumptions: (1) we have a nested sequence of action sets $\A_t \subseteq \A_{t+1}$ with $|\A_T| \leq K$; and (2) one of the best/near-optimal arm belongs to $\A_1$. How to remove/weaken these assumptions is left to future work.
We also remark that, after our initial (arXiv) publication, \citet{marinov2021pareto} established the Pareto frontier for general contextual bandits, providing a negative answer to open problems raised in \citet{foster2020open}.

\subsubsection*{Acknowledgements}
We thank anonymous reviewers for helpful comments. 
This work is partially supported by NSF grant 1934612 and ARMY MURI grant W911NF-15-1-0479.

\bibliographystyle{plainnat}
\bibliography{refs}


\clearpage
\appendix

\thispagestyle{empty}

\onecolumn \makesupplementtitle

\section{OMITTED PROOFS FOR SECTION \ref{sec:lower_bound}}

Besides specific treatments for linear bandits (e.g., the lower bound construction for model selection), our proofs for this section largely follow the ones developed in \citet{hadiji2019polynomial, zhu2020regret}. We provide details here for completeness.

\subsection{Proof of \cref{thm:lower_bound}}
\label{app:lower_boud}
We consider $K+1$ linear bandit instances such that each is characterized by a reward vector $\theta_i \in \R^{d}$, $0 \leq i \leq K$, with different intrinsic dimensions $d_\star$ (or equivalently $\alpha$). For any action $a \in \R^d$, we obtain a reward $r = \langle \theta_i, a \rangle + \eta$ where $\eta$ is an independent $(1/2)$-sub-Gaussian noise. Time horizon $T$ is fixed and the ambient dimension $d$ is assumed to be large enough to avoid some trivial conflicts in the following construction (e.g., we need $d \geq T^\alpha$ to construct $\theta_i$) . For any $0 \leq \alpha^\prime < \alpha \leq 1$ so that $T^\alpha/2 \geq T^{\alpha^\prime}$, we now provide an explicit construction of $\{\theta_i\}_{i=0}^{K}$ as followings, with $\Delta \in \R$ to be specified later. 
\begin{enumerate}
    \item Let $\theta_0 \in \R^{d}$ be any vector such that it is only supported on one of its first $\lfloor T^{\alpha^\prime} \rfloor$ coordinates and $\| \theta_0 \|_2 = \Delta/2$. The regret minimization problem with respect to $\theta_0$ belongs to ${\cal H}_T(\alpha^\prime)$ by construction.
    
    \item For any $i \in [K]$, let $\theta_i = \theta_0 + \Delta \cdot e_{\rho(i)}$ where $e_j$ is the $j$-th canonical base and $\rho(i) = \lfloor T^{\alpha}/2 \rfloor + i$. We set $K = \lfloor T^{\alpha}/2 \rfloor = \Theta(T^\alpha)$ so that the regret minimization problem with respect to any $\theta_i$ belongs to ${\cal H}_T(\alpha)$.
\end{enumerate}
We consider a common \emph{fixed} action set ${\cal A} = \{a_i\}_{i=0}^K  = \{\theta_0 / \norm*{\theta_0}\} \cup \{ e_{\rho(i)} \}_{i=1}^{K}$ for all regret minimization problems (we set $a_0 = \theta_0 / \norm*{\theta_0}$ and $a_i = e_{\rho(i)}$ for convenience). We could notice that $a_0$ is the best arm with respect to $\theta_0$, which has expected reward $\Delta/2$; and $a_i$ is the best arm with respect to $\theta_i$, which has expected reward $\Delta$.

\begin{remark}
\label{rm:expressive_lower_bound}
    The action set $\A$ can be made expressive by augmenting the action set with an all-zero action. The all-zero action will not affect our analysis since it always has zero expected reward.
\end{remark}

\begin{remark}
    One can also add other canonical bases into the action set $\cA$ so that $\curly*{\theta_i}_{i=1}^K$ becomes the unique reward vector for corresponding problems. These additional actions will not affect our analysis as well since they all have zero expected reward.
\end{remark}

For any $t \in [T]$, the tuple of random variables $H_t = (A_1, X_1, \dots, A_t, X_t)$ is the outcome of an algorithm interacting with an bandit instance up to time $t$. Let $\Omega_t = \prod_{i=1}^t (\A \times \R)$ and ${\cal F}_t = \mathfrak{B}(\Omega_t)$; one could then define a measurable space $(\Omega_t, {\cal F}_t)$ for $H_t$. The random variables $A_1, X_1, \dots, A_t, X_t$ that make up the outcome are defined by their coordinate projections:
\begin{align}
A_t(a_1, x_1, \dots, a_t, x_t) = a_t \quad \mbox{and} \quad X_t(a_1, x_1, \dots, a_t, x_t) = x_t. \nonumber
\end{align}
For any fixed algorithm/policy $\pi$ and bandit instance $\theta_i$, we are now constructing a probability measure $\P_{i, t}$ over $(\Omega_t, {\cal F}_t)$. Note that a policy $\pi$ is a sequence $(\pi_t)_{t=1}^T$, where $\pi_t$ is a probability kernel from $(\Omega_{t-1}, {\cal F}_{t-1})$ to $(\A, 2^{\A})$ with the first probability kernel $\pi_1(\omega, \cdot)$ being defined arbitrarily over $(\A, 2^{\A})$, to model the selection of the first action. For each $i$, we define another probability kernel $p_{i, t}$ from $(\Omega_{t-1} \times \A, {\cal F}_{t-1} \otimes 2^{\A})$ to $(\R, \mathfrak{B}(\R))$ that models the reward. Since the reward is distributed according to ${\cal N}(\theta_i^{\top} a_t, 1/4 )$, we gives its explicit expression for any $B \in \mathfrak{B}(\R)$ as following
\begin{align}
p_{i, t} \big( (a_1, x_1, \dots, a_t), B\big) =  \bigintssss_B \sqrt{\frac{2}{\pi}} \exp \big( - 2 (x-\theta_i^{\top} a_t ) \big) dx .\nonumber 
\end{align}
The probability measure over $\P_{i, t}$ over $(\Omega_t, {\cal F}_t)$ could then be define recursively as $\P_{i, t} = p_{i, t} \big( \pi_{t} \P_{i, t-1} \big)$. We use $\E_i$ to denote the expectation taken with respect to $\P_{i, T}$. We have the following lemmas.

\begin{lemma}[\citet{lattimore2020bandit}]
	\label{lm:KL_decomposition}
	\begin{align}
	\kl \left( \P_{0, T}, \P_{i, T} \right) = \E_{0} \left[ \sum_{t=1}^T \kl\left( {\cal N}(\theta_0^{\top} A_t, 1/4), {\cal N} \left( \theta^{\top}_i A_t, 1/4 \right) \right) \right]. \label{eq:KL_decomposition}
	\end{align}
\end{lemma}
\begin{lemma}[\cite{hadiji2019polynomial}]
	\label{lm:pinsker}
	Let $\P$ and $\Q$ be two probability measures. For any random variable $Z \in [0, 1]$, we have \begin{align}
	|\E_{\P}[Z] - \E_{\Q}[Z]| \leq \sqrt{\frac{\kl(\P, \Q)}{2}}. \nonumber
	\end{align}
\end{lemma}

\lowerBound*

\begin{proof}
	Let $N_{i}(T) = \sum_{t=1}^T \mathds{1}\left( A_t = a_i\right)$ denote the number of times the algorithm $\pi$ selects arm $a_i$ up to time $T$. Let $R_{i, T}$ define the expected regret achieved by algorithm $\pi$ interacting with the bandit instance ${\theta}_i$. Based on the construction of bandit instances, we have 
	\begin{align}
	R_{0, T} \geq \frac{\Delta}{2} \sum_{i=1}^K \E_{0} \left[ N_{i}(T) \right], \label{eq:regret_0}
	\end{align}
	and for any $i \in [K]$
	\begin{align}
	R_{i, T} \geq \frac{\Delta}{2} \left( T - \E_{i} [N_{i}(T)] \right) = \frac{T\Delta}{2} \left( 1- \frac{\E_{i} [N_{i}(T)]}{T} \right). \label{eq:regret_i}
	\end{align}
	According to \cref{lm:KL_decomposition} and the calculation of $\kl$-divergence between two Gaussian distributions, we further have
	\begin{align}
	\kl(\P_{0, T}, \P_{i, T}) & = \E_{0} \left[ \sum_{t=1}^T \kl\left( {\cal N}(\theta_0^{\top}A_t, 1/4), {\cal N} \left( \theta_i^{\top}A_t, 1/4 \right) \right) \right] \nonumber\\
	& = \E_{0} \left[ \sum_{t=1}^T 2 \left\langle \theta_i - \theta_0, A_t \right \rangle^2 \right] \nonumber \\
	& = 2 \E_{0} \left[ N_{i}(T) \right] \Delta^2, \label{eq:kl_difference}
	\end{align}
	where \cref{eq:kl_difference} comes from the fact that $\theta_i = \theta_0 + \Delta \cdot e_{\rho(i)}$ and the only arm in ${\cal A}$ with non-zero value on the $\rho(i)$-th coordinate is $a_i = e_{\rho(i)}$, with $\ang*{\theta_i - \theta_0, a_i} = \Delta$.
	
	We now consider the average regret over $i \in [K]$:
	\begin{align}
	\frac{1}{K} \sum_{i=1}^K R_{i, T} &  \geq  \frac{T \Delta}{2} \left(1 - \frac{1}{K} \sum_{i=1}^K \frac{\E_i [N_{i}(T)]}{T} \right) \nonumber \\
	& \geq \frac{T \Delta}{2} \left(1- \frac{1}{K}\sum_{i=1}^K \left(\frac{\E_0 [N_{i}(T)]}{T} + \sqrt{\frac{\kl(\P_{i, T}, \P_{0, T})}{2}}  \right) \right) \label{eq:ave_regret_pinsker} \\
	& = \frac{T \Delta}{2} \left(1- \frac{1}{K} \frac{\sum_{i=1}^K \E_0 [N_{i}(T)]}{T} - \frac{1}{K}\sum_{i=1}^K \sqrt{{\E_{0} \left[ N_{i}(T) \right] \Delta^2}}   \right) \label{eq:ave_regret_decomposition} \\
	& \geq \frac{T \Delta}{2} \left(1 - \frac{1}{K} - \sqrt{\frac{\sum_{i=1}^K \E_{0} \left[ N_{i}(T) \right] \Delta^2}{K}} \right) \label{eq:ave_regret_concave}\\
	& \geq \frac{T \Delta}{2} \left(1 - \frac{1}{K} - \sqrt{\frac{2 \Delta R_{0, T}}{K}} \right) \label{eq:ave_regret_0} \\
	& \geq \frac{T \Delta}{2} \left(\frac{1}{2} - \sqrt{\frac{2 \Delta B}{K}} \right), \label{eq:ave_regret}
	\end{align}
	where \cref{eq:ave_regret_pinsker} comes from applying \cref{lm:pinsker} with $Z = {N_{i}(T)}/{T}$ and $\P = \P_{i, T}$ and $\Q = \P_{0, T}$; \cref{eq:ave_regret_decomposition} comes from \cref{lm:KL_decomposition}; \cref{eq:ave_regret_concave} comes from concavity of $\sqrt{\cdot}$; \cref{eq:ave_regret_0} comes from \cref{eq:regret_0}; and finally \cref{eq:ave_regret} comes from the fact that $K \geq 2$ by construction and the assumption that $R_{0, T} \leq B$.
	
	To obtain a large value for \cref{eq:ave_regret}, one could maximize $\Delta$ while still make sure $\sqrt{2 \Delta B/K} \leq 1/4$. Set $\Delta = 2^{-5}K B^{-1}$, following \cref{eq:ave_regret}, we obtain 
	\begin{align}
	\frac{1}{K} \sum_{i=1}^K R_{i, T} & \geq 2^{-8} T K B^{-1} \nonumber \\
	& = 2^{-8} T \left\lfloor T^{\alpha}/2  \right\rfloor B^{-1} \label{eq:ave_regret_K}\\
	& \geq 2^{-10} T^{1 + \alpha} B^{-1}, \label{eq:ave_regret_final}
	\end{align}
	where \cref{eq:ave_regret_K} comes from the construction of $K$; and \cref{eq:ave_regret_final} comes from the assumption that $\lfloor T^{\alpha}/2  \rfloor \geq T^{\alpha}/4$.
	
	It is clear that any action $a \in \A$ satisfies $\norm*{a} \leq 1$ by construction, we now only need to make sure that $\norm*{\theta_i} \leq 1$ as well. Notice that $\norm*{\theta_i} \leq \sqrt{5}\Delta/2$ by construction, we only need to make sure $\Delta = 2^{-5}K B^{-1} \leq 2/\sqrt{5}$. Since on one hand $K =\lfloor T^{\alpha}/2 \rfloor \leq T^\alpha$, and on the other hand $T^{\alpha} \leq B$ by assumption, we have $\Delta = 2^{-5} K B^{-1} \leq 2^{-5} < 2/\sqrt{5}$, as desired.
\end{proof}

\subsection{Proof of \cref{thm:rate_lower_bound}}

\begin{lemma} 
	\label{lm:rate_comparison}
	Suppose an algorithm achieves rate function $\theta(\alpha)$ on ${\cal H}_T (\alpha)$, then for any $0 < \alpha \leq 1$ such that $\alpha \leq \theta(0)$, we have 
	\begin{align}
	\theta(\alpha) \geq 1 + \alpha - \theta(0). \label{eq:rate_func}
	\end{align}
\end{lemma}
\begin{proof}
	Fix $0 \leq \alpha \leq \theta(0)$. For any $\epsilon > 0$, there exists constant $c_1$ and $c_2$ such that
	\begin{align}
	\sup_{\omega \in {\cal H}_T(0)} R_T \leq c_1 T^{\theta(0) + \epsilon} \quad \mbox{and} \quad \sup_{\omega \in {\cal H}_T(\alpha)} R_T \leq  c_2 T^{\theta(\alpha) + \epsilon}, \nonumber 
	\end{align}
	for sufficiently large $T$.
	Let $B = \max\{c_1,1\} \cdot T^{\theta(0) + \epsilon}$, we could see that $T^{\alpha} \leq T^{\theta(0)} \leq B$ holds by assumption. For $T$ large enough, the condition $\lfloor T^\alpha/2 \rfloor \geq \max\{ T^\alpha/4, T^0, 2\}$ of \cref{thm:lower_bound} holds, and we then have
	\begin{align}
	c_2 T^{\theta(\alpha) + \epsilon} \geq 2^{-10} T^{1 + \alpha} \left( \max\{c_1,1\} \cdot T^{\theta(0) + \epsilon} \right)^{-1} = 2^{-10} T^{1+\alpha - \theta(0) - \epsilon} /\max\{c_1,1\} .\nonumber
	\end{align}
	For $T$ sufficiently large, we then must have
	\begin{align}
	\theta(\alpha) + \epsilon \geq 1 + \alpha - \theta(0) - \epsilon .\nonumber
	\end{align}
	Let $\epsilon \rightarrow 0$ leads to the desired result.
\end{proof}

\thmRateLowerBound*

\begin{proof}
For any adaptive rate function $\theta$ achieved by an algorithm, we first notice that $\theta(\alpha) \geq \theta(\alpha^\prime)$ for any $0 \leq \alpha^\prime \leq \alpha \leq 1$ as ${\cal H}_T(\alpha^\prime) \subseteq {\cal H}_T(\alpha)$, which also implies $\theta(\alpha) \geq \theta(0)$. From \cref{lm:rate_comparison}, we further obtain $\theta(\alpha) \geq 1 + \alpha - \theta(0)$ if $0 < \alpha \leq \theta(0)$. Thus, for any $\alpha \in (0, \theta(0)]$, we have 
\begin{align}
\theta(\alpha) \geq \max \{ \theta(0), 1 + \alpha - \theta(0) \}. \label{eq:rate_comparison}
\end{align}
Note that this indicates $\theta(\theta(0)) = 1$ since we trivially have $R_T \leq T$. For any $\alpha \in [\theta(0), 1]$, we have $\theta(\alpha) \geq \theta(\theta(0)) = 1$, which also leads to $\theta(\alpha) = 1$ for $\alpha \in [\theta(0), 1]$. To summarize, we obtain the desired result in \cref{eq:rate_lower_bound}. We have $\theta(0) \in [1/2, 1]$ as the minimax optimal rate among problems in ${\cal H}_T(0)$ is $1/2$ \citep{chu2011contextual}.
\end{proof}

\section{OMITTED PROOFS FOR SECTION \ref{sec:adaptivity}}

\subsection{The Virtual-Mixture Arm}

The expected reward of virtual mixture-arm $\widetilde{\nu}_j$ can be expressed as the total expected reward obtained in iteration $j$ divided by the corresponding time horizon $\Delta T_j$:
\begin{align}
    \widetilde{\mu}_j 
     = \E [\widetilde{\nu}_j] = \E \left[ \sum_{t  \text{ in iteration $j$}} X_t  \right] / \Delta T_j 
     =  \langle \theta_\star, a_\star \rangle -  R_{\Delta T_j}/\Delta T_j \in [-1, 1], \label{eq:expected_reward_virtual} 
\end{align}
where we use $R_{\Delta T_j}$ to denote the expected regret suffered in iteration $j$. Let $X_t$ be the reward obtained by pulling the virtual arm $\widetilde{\nu}_j$ (with $A_t$ being the feature representation of the drawn real arm), we then know that $X_t - \widetilde{\mu}_j$ is $\sqrt{2}$-sub-Gaussian since $X_t - \widetilde{\mu}_j = \left(X_t - \langle \theta_\star, A_t \rangle \right) + \left(\langle \theta_\star, A_t \rangle - \widetilde{\mu}_j  \right) = \eta_t + \left(\langle \theta_\star, A_t \rangle - \widetilde{\mu}_j  \right)$: $\eta_t$ is $1$-sub-Gaussian by assumption and $\left(\langle \theta_\star, A_t \rangle - \widetilde{\mu}_j  \right)$ is $1$-sub-Gaussian due to boundedness $\langle \theta_\star, A_t \rangle \in [-1, 1]$ and $\E [\langle \theta_\star, A_t \rangle] = \widetilde{\mu}_j$.

\subsection{Modifications of \linucb}
\label{appendix:modified_linucb}

Recall that, under \cref{assumption:action_set}, the linear reward structure is preserved in the modified linear bandit problem that \linucb will be working on in \cref{alg:linucbPlus}. Two main differences in the modified linear bandit problem from the original setting considered in \cite{chu2011contextual} are: (1) we will be working with $\sqrt{2}$-sub-Gaussian noise while they deal with strictly bounded noise; and (2) the norm of our reward parameter, i.e., $\| {\theta}_\star^{\langle d_i \rangle} \|$, could be as large as $1 + (p - 1) = p = \lceil \log_2(T^\beta) \rceil \leq \log_2(T) + 1 \leq 2 \log T $ when $T \geq 2$. 

To reduce clutters, we consider a $d$ dimensional linear bandit with time horizon $T$ and $K$ actions. We consider the reward structure $X_t = \langle \theta_\star, A_t \rangle + \eta_t$, where $\eta_t$ is an independent $\sqrt{2}$-sub-Gaussian noise, $\| \theta_\star\| \leq 2 \log T$ and $\| A_t \| \leq 1$.
The following \cref{thm:modified_linucb} takes care of these changes.

\begin{theorem}
    \label{thm:modified_linucb}
    For the modified setting introduced above, run \linucb with $\alpha = 2\sqrt{ \log (2TK/\delta)}$ leads to an upper bound 
    \begin{align}
        {O}\left( \log^2\left( K T \log ( T) / \delta \right) \cdot \sqrt{d  T } \right) \nonumber 
    \end{align}
    on the (pseudo) random regret with probability at least $1-\delta$.
\end{theorem}

\begin{corollary}
	\label{corollary:linucb}
    For the modified setting introduced above, run \linucb with $\alpha = 2\sqrt{ \log (2T^{3/2}K)}$ leads to an upper bound 
    \begin{align}
        {O}\left( \log^2\left( K T \log ( T)  \right) \cdot \sqrt{d  T } \right) \nonumber 
    \end{align}
    on the expected regret.
\end{corollary}
\begin{proof}
	One can simply combine the result in \cref{thm:modified_linucb} with $\delta = 1/\sqrt{T}$.
\end{proof}

It turns out that in order to prove \cref{thm:modified_linucb}, we mainly need to modify Lemma 1 in \cite{chu2011contextual}, and the rest of the arguments go through smoothly. The changed exponent on the logarithmic term is due to $\| \theta_\star\| \leq 2 \log T$. We introduce the following notations. Let
\begin{align*}
    V_0 = I  \quad \text{and} \quad V_t = V_{t-1} + A_t A_t^\top
\end{align*}
denote the design matrix up to time $t$; and let
\begin{align*}
    \widehat{\theta}_t = V_t^{-1}\sum_{i=1}^t A_i X_i
\end{align*}
denote the estimate of $\theta_\star$ at time $t$. 

\begin{lemma}(modification of Lemma 1 in \cite{chu2011contextual})
    Suppose for any fixed sequence of selected actions $\{A_i\}_{i \leq t}$ the (random) rewards $\{X_i\}_{i\leq t}$ are independent. Then we have 
    \begin{align}
        \P \left( \forall A_{t+1} \in \A_{t+1}: \vert \langle \widehat{\theta}_{t} - \theta_\star , A_{t+1} \rangle \vert \leq (\alpha + 2\log T) \sqrt{A_{t+1}^\top V_t^{-1} A_{t+1}} \right) \geq 1-\delta/T. \label{eq:lm_bound} 
    \end{align}
    
\end{lemma}

\begin{remark}
	The requirement of (conditional) independence is guaranted by the \suplinucb algorithm introduced in \cite{chu2011contextual}, and is not satisfied by the vanilla \linucb: the reveal/selection of a future arm $A_{t+1}$ makes previous rewards $\{X_i\}_{i \leq t}$ dependent. See Remark 4 in \cite{han2020sequential} for a detailed discussion.
\end{remark}

\begin{proof}
    For any fixed $A_t$, we first notice that 
    \begin{align}
        \abs  {\ang {\widehat{\theta}_{t} - \theta_\star , A_{t+1} } } & = \abs { A^{\top}_{t+1} V_t^{-1}\sum_{i=1}^t A_i X_i - A^{\top}_{t+1} \theta_\star}\nonumber \\
        & = \abs{ A^{\top}_{t+1} V_t^{-1}\sum_{i=1}^t A_i X_i - A^{\top}_{t+1} V_t^{-1} \left( I + \sum_{i=1}^t A_i A_i^\top \right)\theta_\star } \nonumber \\
        & \leq  \abs{ \sum_{i=1}^t A^\top_{t+1} V^{-1}_t A_i \left( X_i - A_i^{\top} \theta_\star \right)} + \abs{A^{\top}_{t+1} V_t^{-1} \theta_\star} \nonumber \\
        & \leq \abs{ \sum_{i=1}^t A^\top_{t+1} V^{-1}_t A_i \left( X_i - A_i^{\top} \theta_\star \right)} +  \norm{A^{\top}_{t+1} V_t^{-1} } \cdot \norm{\theta_\star}. \label{eq:lm_bound_two_terms}
    \end{align}
We next bound the two terms in \cref{eq:lm_bound_two_terms} seperately.

For the first term in \cref{eq:lm_bound_two_terms}, since $\left( X_i - A_i^\top \theta_\star \right)$ is $\sqrt{2}$-sub-Gaussian and $\{X_i\}_{i \leq t}$ are independent, we know that $\sum_{i=1}^t A^\top_{t+1} V^{-1}_t A_i \left( X_i - A_i^{\top} \theta_\star \right)$ is $\left( \sqrt{2 \sum_{i=1}^t \left( A^\top_{t+1} V^{-1}_t A_i \right)^2} \right)$-sub-Gaussian. Since 
\begin{align}
    \sqrt{\sum_{i=1}^t \left( A^\top_{t+1} V^{-1}_t A_i \right)^2}  & 
    = \sqrt{\sum_{i=1}^t A_{t+1}^\top V_t^{-1} A_i A_i^\top V_t^{-1} A_{t+1}} \nonumber \\
    & \leq \sqrt{A_{t+1}^\top V_t^{-1} \left( I + \sum_{i=1}^t A_i A_i^\top \right) V_t^{-1} A_{t+1}} \nonumber \\
    & = \sqrt{A_{t+1}^\top V_t^{-1} A_{t+1}}, \nonumber 
\end{align}
according to a standard Chernoff-Hoeffding bound, we have 
\begin{align}
    \P \left( \abs{ \sum_{i=1}^t A^\top_{t+1} V^{-1}_t A_i \left( X_i - A_i^{\top} \theta_\star \right)} \geq \alpha \sqrt{A_{t+1}^\top V_t^{-1} A_{t+1}} \right) & \leq 2 \exp\left( -\frac{\alpha^2}{4} \right) \nonumber \\
    & = \frac{\delta}{TK}, \label{eq:lm_bound_alpha}
\end{align}
where \cref{eq:lm_bound_alpha} is due to $\alpha = 2\sqrt{\log (2TK/\delta)}$.

For the second term in \cref{eq:lm_bound_two_terms}, we have 
\begin{align}
   \norm{A^{\top}_{t+1} V_t^{-1} } \cdot \norm{\theta_\star} & \leq 2 \log T \, \sqrt{A^\top_{t+1} V^{-1}_t I V^{-1}_t A_{t+1}} \label{eq:lm_bound_norm} \\
   & \leq 2 \log T \, \sqrt{A^\top_{t+1} V^{-1}_t \left(I + \sum_{i=1}^t A_i A_i^\top \right) V^{-1}_t A_{t+1}} \nonumber \\
   & =  2 \log T \,\sqrt{A_{t+1}^\top V_t^{-1} A_{t+1}} . \nonumber 
\end{align}
where \cref{eq:lm_bound_norm} comes from the fact that $\norm{\theta_\star} \leq 2 \log T$.

The desired result in \cref{eq:lm_bound} follows from a union bound argument together with the two upper bounds derived above. 
\end{proof}

\begin{remark}
	Technically, regret guarantees are for a more complicated version of \linucb that ensures statistical independence \citep{chu2011contextual}. However, as recommended by \cite{chu2011contextual}, we will use the more practical \linucb as our subroutine. 
\end{remark}

\subsection{Notations and Preliminaries for Analysis of \linucbPlus}
\label{app:preliminary_linucbPlus}

We provide some notations and preliminaries for analysis of \linucbPlus that will be used in the following two subsections, i.e., the proofs of \cref{lm:linucb_learning_error} and \cref{thm:linucbPlus}.

We define $T_i = \sum_{j = 1}^{i} \Delta T_j$ so that the $i$-th iteration of \linucbPlus goes from $T_{i-1} + 1$ to $T_i$. We first notice that \cref{alg:linucbPlus} is a valid algorithm in the sense that it selects an arm $A_t$ for any $t \in [T]$, i.e., it does not terminate before time $T$: the argument is clearly true if there exists $i \in [p]$ such that $\Delta T_i = T$; otherwise, we can show that 
	\begin{align*}
	    T_p  = \sum_{i=1}^p \Delta T_i = 2(2^{2p} - 1) \geq 2^{2p} \geq T,
	\end{align*}
    for all $\beta \in [1/2, 1]$. 

We use $R_{\Delta T_i} = \Delta T_i \cdot \mu_\star - \E [\sum_{t = T_{i-1}+1}^{T_i} X_t ]$ to denote the expected cumulative regret at iteration $i$. Let ${\cal F}_{i}$ denote the information collected up to the end of iteration $i$, we further use $R_{\Delta T_i \vert {\cal F}_{i-1}}$ to represent the expected regret conditioned on ${\cal F}_{i-1}$ and have $\E[R_{\Delta T_i \vert {\cal F}_{i-1}}] = R_{\Delta T_i}$.

In the modified linear bandit problem at each iteration $i$, we will be applying \linucb with respect to a $d_i +i -1$ dimensional problem with an action set $\A^{\ang{d_i}}$ such that $\abs*{\A^{\ang{d_i}}} \leq K + i -1$. Let $a^{\langle d_i \rangle }_{\star} = \argmax_{a \in {\cal A}^{\langle d_i \rangle}} \{\langle \theta^{\langle d_i \rangle}_{\star}, a \rangle \}$ denote the best arm in the $i$-th iteration. Applying \cref{eq:regret_decomposition} on $R_{\Delta T_i \vert {\cal F}_{i-1}}$ leads to 
\begin{align}
	\label{eq:regret_decomposition_internal}
	R_{\Delta T_i\vert {\cal F}_{i-1}} = \Delta T_i \cdot  \left(\langle \theta_\star, a_\star \rangle - \langle \theta_{\star}^{\langle d_i \rangle}, a_{\star}^{\langle d_i \rangle} \rangle \right) +  \E \left[ \sum_{t=T_{i-1}+1}^{T_i} \langle \theta_{\star}^{\langle d_i \rangle}, a_{\star}^{\langle d_i \rangle} - A_t \rangle \, \bigg\vert \, {\cal F}_{i-1} \right] ,
\end{align}
where $A_t \in {\cal A}^{\langle d_i \rangle}$ and $\langle \theta_{\star}^{\langle d_i \rangle}, A_t \rangle$ represents the expected reward of pulling arm $A_t$.

\subsection{Proof of \cref{lm:linucb_learning_error}}
The proof of \cref{lm:linucb_learning_error} follows the notations and preliminaries introduced in \cref{app:preliminary_linucbPlus}.

\begin{restatable}{lemma}{lmLinucbLearningError}
    \label{lm:linucb_learning_error}
    At each iteration $i \in [p]$, the learning error suffered from subroutine \linucb is upper bounded by ${O} \paren*{\log^{5/2} \left( KT\log T\right) \cdot T^\beta}$.
\end{restatable}

\begin{proof}

    We focus on the second term in \cref{eq:regret_decomposition_internal}, i.e., the (conditional) learning error during iteration $i$. Conditioning on $\F_{i-1}$, both $\theta^{\langle d_i \rangle}_\star$ and $a_\star^{\ang{d_i}}$ can be treated as fixed quantities. Applying the regret bound in \cref{corollary:linucb}, we have:
	\begin{align}
	\E \left[ \sum_{t=T_{i-1}+1}^{T_i} \langle \theta_{\star}^{\langle d_i \rangle}, a_{\star}^{\langle d_i \rangle} - A_t \rangle \, \bigg\vert \, {\cal F}_{i-1} \right] 
	& = {O} \left( \log^2 \left( (K+i-1) \Delta T_i \log (\Delta T_i) \right) \cdot \sqrt{(d_i +i -1 ) \Delta T_i  } \right) \label{eq:regret_learning_0} \\
	& = {O} \left( \log^2 \left( (K+p) \Delta T_i \log (\Delta T_i) \right) \cdot  \sqrt{(d_i+p) \Delta T_i}  \right) \label{eq:regret_learning_1}\\
	& = {O} \left(\log^2 \left( (K+p) T \log T \right)  \cdot \sqrt{2^{2p+2} + p T} \right) \label{eq:regret_learning_2} \\
	& = {O} \left(\log^2 \left( K T \log T \right)  \cdot \sqrt{T^{2\beta} + \log T \cdot T} \right) \label{eq:regret_learning_3} \\
	& = {O} \left(\log^{5/2} \left( KT\log T\right) \cdot T^\beta  \right), \label{eq:regret_learning}
	\end{align}
	where \cref{eq:regret_learning_0} comes from the guarantee of \linucb in \cref{corollary:linucb}; \cref{eq:regret_learning_1} uses the fact that $i \leq p $; \cref{eq:regret_learning_2} comes from the definition of $d_i$ and $\Delta T_i$; \cref{eq:regret_learning_3} comes from the fact that $p = \ceil*{ \log_2 T^\beta}$; \cref{eq:regret_learning} comes from trivially bounding $\sqrt{T^{2\beta} + \log T \cdot T} = O((\log T)^{1/2} \cdot T^\beta)$.\footnote{One can improve the bound to $\sqrt{T^{2\beta} + \log T \cdot T} = O(T^\beta)$ in many cases, e.g., when $\beta > 1/2$. We mainly focus on the polynomial terms here.} The desired result follows from taking another expectation over randomness in $\F_{i-1}$.
\end{proof}

\subsection{Proof of \cref{thm:linucbPlus}}
\label{app:thm_linucbPlus}
The proof of \cref{thm:linucbPlus} follows the notations and preliminaries introduced in \cref{app:preliminary_linucbPlus}.

\thmLinucbPlus*

\begin{proof}

    When $\alpha \geq \beta$, one could see that \cref{thm:linucbPlus} trivially holds since $T^{1 + \alpha - \beta} \geq T$. In the following, we only consider the case when $\alpha < \beta$.

	Taking expectation on \cref{eq:regret_decomposition_internal} and combining the result in \cref{lm:linucb_learning_error}, we obtain
	\begin{align}
	\label{eq:regret_decomposition_internal_expectation}
	R_{\Delta T_i}  =  \Delta T_i \cdot  \E \left[ \left( \langle \theta_\star, a_\star \rangle -\langle \theta_{\star}^{\langle d_i \rangle}, a_{\star}^{\langle d_i \rangle} \rangle  \right) \right]  + {O} \left( \log^{5/2} \left( KT\log T\right) \cdot T^\beta \right).
	\end{align}

	We now focus on the first term, i.e., the expected approximation error over the $i$-th iteration. Notice that, according to the definition of $a^{\langle d_i \rangle}_{\star}$ and $\theta^{\langle d_i \rangle}_\star$, we have $\langle \theta_{\star}^{\langle d_i \rangle}, a_{\star}^{\langle d_i \rangle} \rangle = \langle \theta_\star, a_\star \rangle$ if $d_i \geq d_\star$, i.e., the optimal arm is contained in the action set $\A^{\ang{d_i}}$. Let $i_\star \in [p]$ be the largest integer such that $d_{i_\star} \geq d_{\star}$, we then have that, for any $i \leq i_\star$ and in particular for $i = i_\star$,
	\begin{align}
	 R_{\Delta T_{i}} =  {O} \left(T^\beta \log^{5/2} \left( KT\log T\right)\right). \label{eq:regret_i0}
	\end{align}
	
	In the case when $\Delta T_{i_\star} = \min \{2^{p+i_\star}, T \}=T$ or $i_\star = p$, we know that \linucbPlus will in fact stop at a time step no larger than $T_{i_\star}$ (since the allowed time horizon is $T$), and incur no regret in iterations $i > i_\star$. In the following, we only consider the case when $\Delta T_{i_\star} = 2^{p+i_\star}$ and $i_\star < p$. To incooperate another possible corner case when $d_{i_\star} = \min\{2^{p+2-i_\star}, d\} = d$, we consider $d_{i_\star + 1} = 2^{p + 1 - i_\star} < d_{i_\star}$. As a result, we have $d_{i_\star} \Delta T_{i_\star} > d_{i_\star + 1} \Delta T_{i_\star} = 2^{2p+1}$, which leads to
	\begin{align}
	\Delta T_{i_\star} > \frac{2^{2p + 1}}{d_{i_\star}} > \frac{2^{2p }}{d_\star} = \frac{2^{2p }}{T^{\alpha}}, \label{eq:Ti0}
	\end{align}
	where \cref{eq:Ti0} comes from the fact that $ d_{i_\star} < 2 d_\star$ according to the definition of $i_\star$.\footnote{We will have $\Delta T_{i_\star}  \geq 2^{2p+1}/T^\alpha > 2^{2p}/T^\alpha$ if $d_{i_\star} = \min\{2^{p+2-i_\star}, d\}  = 2^{p+2-i_\star}$.}
	
	We now analysis the expected approximation error for iteration $i > i_\star$. Since the sampling information during $i_\star$-th iteration is summarized in the virtual mixture-arm $\widetilde{\nu}_{i_\star}$, and its representation $\widetilde{\nu}^{\langle d_i \rangle}_{i_\star}$ is added to ${\cal A}^{\langle d_i \rangle}$. For any $i > i_\star$, we then have 
	\begin{align}
    \Delta T_i \cdot  \E \left[ \left( \langle \theta_\star, a_\star \rangle -\langle \theta_{\star}^{\langle d_i \rangle}, a_{\star}^{\langle d_i \rangle} \rangle \right) \right]  & \leq \Delta T_i \cdot \E \left[\left(\langle \theta_\star, a_\star \rangle - \langle \theta_{\star}^{\langle d_i \rangle}, \widetilde{\nu}^{\langle d_i \rangle}_{i_\star} \rangle \right) \right]  \nonumber  \\
    & = \Delta T_i \cdot \left(\langle \theta_\star, a_\star \rangle -  \widetilde{\mu}_{i_\star} \right) \label{eq:regret_approximation_equivalent_form}  \\
	& = \frac{\Delta T_i}{\Delta T_{i_\star}}\cdot  R_{\Delta T_{i_\star}} \label{eq:regret_approximation_expected_reward} \\
	& = \frac{\Delta T_i}{\frac{2^{2p}}{T^\alpha }} \cdot {O} \left(\log^{5/2} \left( KT\log T\right) \cdot T^\beta  \right)\label{eq:regret_approximation_i_star} \\
	& = \frac{ {O} \left(  \log^{5/2} \left( KT\log T\right)  \cdot T^{1 + \alpha + \beta} \right) }{{2^{2p}}} \label{eq:regret_approximation_important} \\
	& = {O} \left( \log^{5/2} \left( KT\log T\right) \cdot T^{1  + \alpha - \beta} \right), \label{eq:regret_approximation}
	\end{align}
	where \cref{eq:regret_approximation_equivalent_form} comes from the formulation of the modified linear bandit problem; \cref{eq:regret_approximation_expected_reward} comes from that fact that $\widetilde{\mu}_j = \E [\widetilde{\mu}_{j \vert \F_j}]=\langle \theta_\star, a_\star \rangle -  R_{\Delta T_j}/\Delta T_j$ derived from \cref{eq:expected_reward_virtual};
	\cref{eq:regret_approximation_i_star} comes from the bound in \cref{eq:regret_i0} with $i=i_\star$; \cref{eq:regret_approximation_important} comes from the fact that $\Delta T_i \leq T$ and some rewriting; \cref{eq:regret_approximation} comes from the fact that $p = \lceil \log_2 T^{\beta} \rceil \geq \log_2 T^{\beta}$.
	
Combining \cref{eq:regret_approximation} and \cref{eq:regret_decomposition_internal_expectation} for cases when $i > i_\star$ (or the corner case algorithm stops before $T_{i_\star}$ and incurs no regret in iterations $i \geq i_\star$), and together with \cref{eq:regret_i0} for cases when $i \leq i_\star$, we have that $\forall i \in [p]$,
\begin{align}
    R_{\Delta T_i}  & =  {O} \left( \log^{5/2} \left( KT\log T\right) \cdot T^{\max\{\beta, 1  + \alpha - \beta \}} \right) . \nonumber 
\end{align}

Since the cumulative regret is non-decreasing in $t$, we have
\begin{align}
	R_{T} & \leq \sum_{i=1}^p R_{\Delta T_i} \nonumber \\
	& = \sum_{i=1}^p {O} \left( \log^{5/2} \left( KT\log T\right) \cdot T^{\max\{\beta, 1  + \alpha - \beta \}} \right) \nonumber \\
    & = {O} \left( \log^{7/2} \left( KT\log T\right) \cdot T^{\max\{\beta, 1  + \alpha - \beta \}} \right),\nonumber
\end{align}
where we use the fact that $p = \lceil \log_2(T^\beta) \rceil = O(\log T)$. Our results follows after noticing $R_T \leq T$ is a trivial upper bound.
\end{proof}

\subsection{Proof of \cref{thm:pareto}}

\pareto*

\begin{proof}
From \cref{thm:linucbPlus}, we know that the rate in \cref{eq:pareto_rate} is achieved by \cref{alg:linucbPlus} with input $\beta$. We only need to prove that no other algorithms achieve strictly smaller rates in pointwise order.

Suppose, by contradiction, we have $\theta^\prime$ achieved by an algorithm such that $\theta^\prime(\alpha) \leq \theta_{\beta}(\alpha)$ for all $\alpha \in [0, 1]$ and $\theta^\prime(\alpha_0) < \theta(\alpha_0)$ for at least one $\alpha_0 \in [0, 1]$. We then must have $\theta^\prime(0) \leq \theta_{\beta}(0) = \beta$. We consider the following two exclusive cases.

\textbf{Case 1 $\theta^\prime(0) = \beta$.} According to \cref{thm:rate_lower_bound}, we must have $\theta^\prime \geq \theta_\beta$, which leads to a contradiction.

\textbf{Case 2 $\theta^\prime(0) = \beta^\prime < \beta$.} According \cref{thm:rate_lower_bound}, we must have $\theta^\prime \geq \theta_{\beta^\prime}$. However, $\theta_{\beta^\prime}$ is not strictly better than $\theta_{\beta}$, e.g., $\theta_{\beta^\prime}(2\beta - 1) =  2\beta - \beta^\prime > \beta = \theta_{\beta}(2\beta - 1)$, which also leads to a contradiction.
\end{proof}

\section{ANALYSIS FOR SECTION \ref{sec:remove_assumption}}
\label{app:remove_assumption}

\subsection{Discussion on \cref{alg:linucbPlus_modified}}
We construct the following two (smoothed) base algorithms \citep{pacchiano2020model} at each iteration of \linucbPlus: (1) a \linucb algorithm that works with truncated feature representations in $\R^{d_i}$, with possible mis-specifications; and (2) a \ucb algorithm that works only with virtual mixture-arms, if there exists any. We use \smoothCorral from \cite{pacchiano2020model} as the master algorithm and always optimally tune it with respect to the \linucb base, i.e., set the learning rate as $\eta = 1/\sqrt{d_i \Delta T_i}$. For iterations such that $d_i \geq d_\star$, the \linucb is the optimal base and we incur $\widetilde{O}(\sqrt{d_i \Delta T_i}) = \widetilde{O}(T^\beta)$ regret; a good enough virtual mixture-arm $\widetilde{\nu}_{i_\star}$ is then constructed as before. For later iterations such that $d_i < d_\star$, \smoothCorral incurs regret $\widetilde{O}(\max \{T^{1+\alpha - \beta}, T^{\beta}\})$ thanks to guarantees of the \ucb base: the $\widetilde{O}(T^{1+\alpha - \beta})$ term is due to the approximation error and the $\widetilde{O}(T^\beta)$ term is due to the learning error. Although the learning error of \ucb is enlarged from $\widetilde{O}(T^{1/2})$ to $\widetilde{O}(T^\beta)$, as \smoothCorral is always tuned with respect to the \linucb base, this won't affect the resulted Pareto optimality.

\subsection{Proof of \cref{thm:linucbPlus_modified}}

\thmLinUCBPlusModified* 

\begin{proof}

At each iteration $i \in [p]$ of \linucbPlus, we applying \smoothCorral as the master algorithm with two smoothed base algorithms: (1) a \linucb algorithm that works with truncated feature representations in $\R^{d_i}$, with possible mis-specifications; and (2) a \ucb algorithm that works only with virtual mixture-arms, if there exists any. The learning rate of \smoothCorral is always optimally tuned with respect to the \linucb base, i.e., $\eta = 1/\sqrt{d_i \Delta T_i}$. Since there are at most $p = O(\log T)$ iterations, we only need to bound the expected regret at each iteration $R_{\Delta T_i}$. As before, we use $i_\star \in [p]$ to denote the largest integer such that $d_{i_\star} \geq d_\star$.

For $i \leq i_\star$, the \linucb base works on a well-specified linear bandit problem. Theorem 5.3 in \cite{pacchiano2020model} gives the following guarantees:
\begin{align*}
    R_{\Delta T_i} = \widetilde{O} \left(\sqrt{\Delta T_i} + \eta^{-1} + \Delta T_i \eta + \Delta T_i d_i \eta \right) = \widetilde{O} \left(\sqrt{d_i \Delta T_i} \right) = \widetilde{O} \left(T^\beta \right).
\end{align*}
Good enough virtual mixture-arm $\widetilde{\nu}_{i_\star}$ is then constructed with conditional expectation $\widetilde{\mu}_{i_\star \vert \F_{i_\star}} = \E [\widetilde{\nu}_{i_\star} \vert \F_{i_\star}] = \langle \theta_\star, a_\star \rangle -  \widehat R_{\Delta T_{i_\star}}/\Delta T_{i_\star}$.

We now analyze the regret incurred for iteration $i > i_\star$. Conditioning on past information $\F_{i-1}$ and let $r(\pi_t)$ denote the (conditional) expected reward of applying policy $\pi_t$, we have 
\begin{align}
    R_{\Delta T_i \vert \F_{i-1}} & = \Delta T_i \cdot   \left( \langle \theta_\star, a_\star \rangle -\widetilde{\mu}_{i_\star \vert \F_{i_\star}} \right) + \E \left[ \sum_{t \text{ in iteration } i} \widetilde{\mu}_{i_\star\vert \F_{i_\star} } - r(\pi_t) \, \bigg\vert \, {\cal F}_{i-1}  \right] \nonumber \\
    & = \Delta T_i \cdot   \left( \langle \theta_\star, a_\star \rangle -\widetilde{\mu}_{i_\star \vert \F_{i_\star}} \right) + \widetilde{O} \left(\sqrt{\Delta T_i} + \eta^{-1} + \Delta T_i \eta + \Delta T_i \eta \right) \nonumber ,
\end{align}
where the second term comes from the guarantee of \smoothCorral with respect to the \ucb base. Taking expectation over randomness in $\F_{i-1}$ leads to
\begin{align}
      R_{\Delta T_i }  & = \widetilde{O}\left(T^{1+\alpha - \beta} \right) + \widetilde{O}\left( T^\beta \right) \nonumber ,
\end{align}
where the first term follows from a similar analysis as in \cref{eq:regret_approximation}, and the second term follows by setting $\eta = 1/\sqrt{d_i \Delta T_i}$. A similar analysis as in \cref{thm:pareto} thus show \cref{alg:linucbPlus_modified} is Pareto optimal, even without \cref{assumption:action_set}.
\end{proof}

\subsection{Discussion on \smoothCorral}
\label{app:corral}

\cite{pacchiano2020model} tackles the model selection problem in linear bandit by applying \smoothCorral with $O(\log d)$ base \linucb learners working with different dimensions $d_i \in \{2^0, 2^1, \dots, 2^{\floor*{\log d}}\}$. Let $d_{i_\star}$ denote the smallest dimension that satisfies $d_{i_\star} \geq d_\star$. With respect to the base \linucb working on the first $d_{i_\star}$ dimensions, \smoothCorral enjoys regret guarantee
\begin{align}
    R_{T} = \widetilde{O} \left(\sqrt{T} + \eta^{-1} + T \eta + T d_\star \eta \right) \nonumber.
\end{align}
\smoothCorral then achieves the rate function in \cref{eq:pareto_rate} by setting the learning rate $\eta = T^{-\beta}$ (and also noticing that $d_\star \leq T^{\alpha}$).

\section{ADDITIONAL EXPERIMENT RESULTS}
\label{app:experiment}

We conduct additional experiments with setups similar to the ones shown in \cref{fig:non_expressive_alpha}, but with different reward parameters $\theta_\star$. 
We set $\theta_\star$ as (the normalized version of) $[\frac{1}{\sqrt{1}}, \frac{1}{\sqrt{2}}, \dots, \frac{1}{\sqrt{d_\star}}, 0, \dots, 0]^\top \in \R^d$ in \cref{fig:decay}; and $\theta_\star$ as (the normalized version of) $[\frac{1}{\sqrt{d_\star}}, \frac{1}{\sqrt{d_\star - 1}}, \dots, \frac{1}{\sqrt{1}}, 0, \dots, 0]^\top \in \R^d$ in \cref{fig:flip}. 
With $\theta_\star$ selected in \cref{fig:decay}, \dynamicBalancing shows comparable performance to \linucbPlus in terms of averaged regret (but with larger variance). \linucbPlus outperforms \dynamicBalancing when $\theta_\star$ is ``flipped'' (i.e., the one used in \cref{fig:flip}) but with the same intrinsic dimension $d_\star$.

\begin{figure}[h]
     \centering
     \subfloat[]{\includegraphics[width=.5\textwidth]{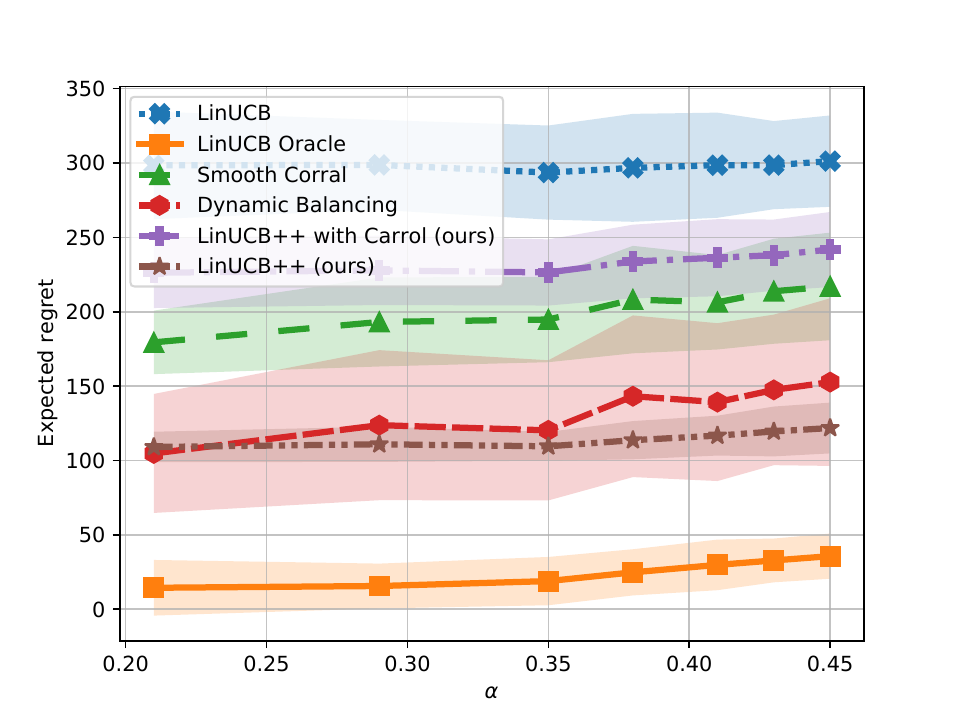}\label{fig:decay}}
     \subfloat[]{\includegraphics[width=.5\textwidth]{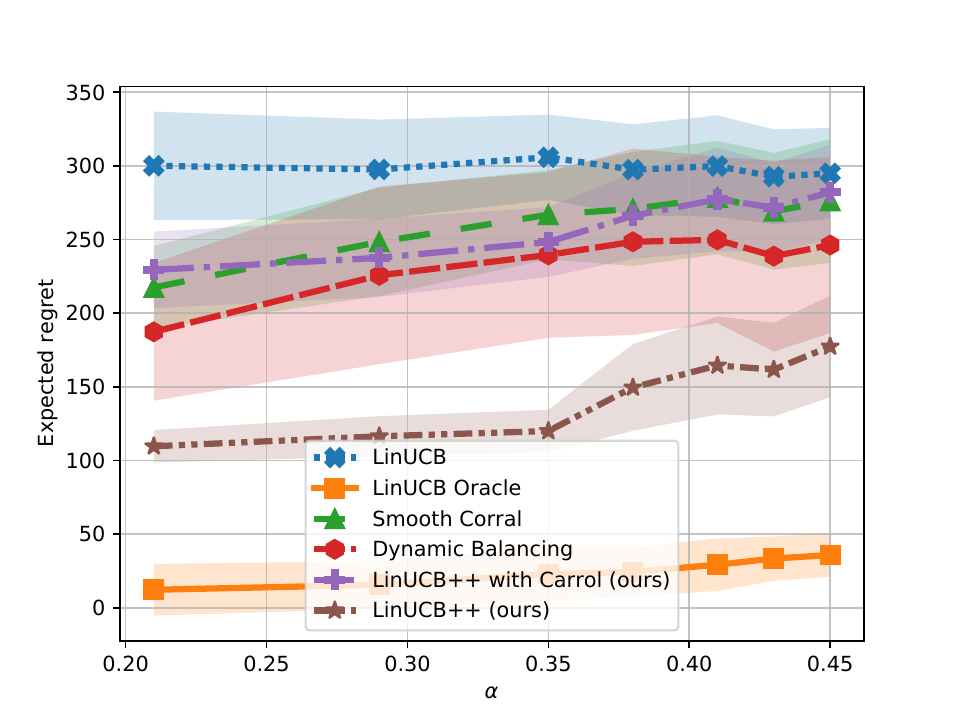}\label{fig:flip}}
     \caption{Similar Experiment Setups to Those Shown in \cref{fig:non_expressive_alpha}, but with Different Reward Parameters $\theta_\star$.}
     \label{fig:additional}
\end{figure}

\end{document}